\documentclass{applemlr}

\usepackage{amsmath}
\usepackage{enumerate}
\usepackage{algorithm}
\usepackage{algpseudocode}
\usepackage{amsfonts}
\usepackage{amsthm}
\usepackage{cleveref}
\usepackage{diagbox}
\usepackage{colortbl}
\usepackage{amssymb}
\usepackage{xspace}
\usepackage{wrapfig}
\usepackage{adjustbox}
\usepackage{tabularx}
\usepackage{booktabs}
\usepackage{mathtools}
\usepackage{tikz}
\usepackage{enumitem}
\usepackage{silence}
\usepackage{dsfont}
\usepackage[table]{xcolor}
\usepackage[dvipsnames]{xcolor}
\usepackage{multirow}
\usepackage{makecell}
\usepackage{xfakebold}

\usepackage{amsmath,amsfonts,bm}

\def\eqref#1{equation~\ref{#1}}

\def\1{\bm{1}}

\DeclareMathAlphabet{\mathsfit}{\encodingdefault}{\sfdefault}{m}{sl}
\SetMathAlphabet{\mathsfit}{bold}{\encodingdefault}{\sfdefault}{bx}{n}

\definecolor{textgray}{HTML}{6E6E73}
\usetikzlibrary{positioning, calc}
\usetikzlibrary{decorations.pathmorphing}

\makeatletter
\patchcmd{\wrong@fontshape}{\@gobbletwo}{}{}{}
\makeatother
\numberwithin{equation}{section}
\makeatletter
\AtBeginDocument{
  \urlstyle{sf}
  
}
\makeatother

\definecolor{light}{RGB}{125, 125, 125}
\crefname{tcb@cnt@pbox}{code}{code}
\Crefname{tcb@cnt@pbox}{Code}{Code}
\crefname{assumption}{assumption}{assumption}
\Crefname{assumption}{Assumption}{Assumptions}

\newtcolorbox[auto counter]{pbox}[2][]{
  colback=white,
  title=Code~\thetcbcounter: #2,
  #1,fonttitle=\sffamily,
  fontupper=\sffamily,
  arc=2pt,
  colframe=bgcolor,
  coltitle=fgcolor,
  colbacktitle=bgcolor,
  toptitle=0.25cm,
  bottomtitle=0.125cm
}

\makeatletter
\newcommand\applefootnote[1]{%
  \begingroup
  \renewcommand\thefootnote{}%
  \renewcommand\@makefntext[1]{\noindent##1}%
  \footnote{#1}%
  \addtocounter{footnote}{-1}%
  \endgroup
}
\makeatother

\definecolor{cverbbg}{gray}{0.90}

\newtheorem{lemma}{Lemma}

\title{Structuring MoE Expert Selection for Agentic Reinforcement Learning}

\author[1,2]{Bolian Li}
\author[1]{Ting-Yao Hu}
\author[1]{Cheng-Yu Hsieh}
\author[1]{Sanjoy Chowdhury}
\author[1]{Oncel Tuzel}
\author[1]{Raviteja Vemulapalli}

\affiliation[1]{Apple}
\affiliation[2]{Purdue University}

\abstract{
    Long-horizon LLM agents are frequently implemented using sparse mixture-of-experts (MoE) models, yet the co-design of agentic behavior and MoE structures remains underexplored. In this work, we comprehensively study the connections between agentic post-training and MoE expert selection. In off-the-shelf MoE models, we observe expert selection exhibits a specialized structure that naturally aligns with agentic trajectories. Specifically, expert routing overlaps more between turns where the agent performs semantically similar operations (e.g., READ, UPDATE) than between turns with differing operations. However, standard RL algorithms ignore this specialization, allowing the MoE routing to go uncontrolled during training, which empirically limit task performance and inference efficiency. To address this, we introduce a hierarchical routing control framework for agentic tasks. We explicitly encourage turn-level expert selections to align with agentic operations while regularizing token-level expert selections to maintain local consistency. To resolve stability issues that arise during post-training with the proposed methods, we further introduce an entropy-gated control mechanism. Overall, our routing control framework achieves over 10-point improvements in success rate on all evaluated benchmarks. These results demonstrate that agentic trajectory structure provides an effective signal for optimizing MoE capacity during RL post-training.
}

\metadata[Correspondence]{\sffamily Bolian Li: \url{bli67@apple.com}
}
\date{\sffamily\today}

\begin{document}

\maketitle

\section{Introduction}
Recent progress in long-horizon interactive agents has been driven by two complementary advances. First, external scaffolding built around LLMs enables complex agentic behaviors, such as multi-step reasoning and making external tool calls~\citep{yao2023react,trivedi2024appworld, shepard2026automationbench}. Second, mixture-of-experts (MoE) architectures allow model capacity to scale efficiently by activating only a fraction of parameters per token~\citep{liu2024deepseek}. This sparse computation is particularly useful for agentic workloads, which generate long trajectories through repeated interactions with environments. Although long-horizon agents are frequently implemented using MoE architectures~\citep{liu2025deepseek, team2026kimi, zeng2026glm}, the co-design of agentic frameworks and MoE routing remains underexplored. While current agentic post-training algorithms optimize the actor policy for downstream performance, they fail to account for how these updates affect expert selection in MoE models.

Agentic trajectories provide a natural structure for understanding this co-design. A trajectory repeatedly alternates between a \emph{thinking} field, a \emph{tool-use} field, and \emph{environment feedback}. We encapsulate the three fields as one \emph{turn} in an agentic trajectory. Semantically, a turn performs an identifiable operation, such as making tool calls that read, create, or update an object. We observe in Fig.~\ref{fig:method_pre_analysis} that the routers of MoE models already exhibit an operation-grouped specialization structure: adjacent tokens within one field tend to reuse experts, while turns with semantically similar operations tend to share similar expert distributions. In contrast, such specialization is not obvious for other grouping criteria like tool names. However, this structure is overlooked in standard MoE-agnostic post-training. Standard RL does not encourage consistency between adjacent tokens or operation-grouped specialization, which limits both task performance and efficiency of MoE-based long-horizon agents, as demonstrated in Fig.~\ref{fig:effect_overview}.

In this work, we hypothesize that a well-structured expert selection mechanism is key to improving agentic task performance. Specifically, aligning expert routing with the semantic structure of the trajectory---by maintaining consistency within a \textit{thinking} or \textit{tool-use} field, and specializing experts for distinct operations---allows the model to utilize its expert capacity more effectively. To explicitly optimize this structure, we introduce a hierarchical routing control framework for expert selection during agentic RL. At the turn level, we maximize the mutual information (MI) between the router distribution, averaged across all tokens within a respective \textit{thinking} or \textit{tool-use} field in the turn, and the corresponding operation label for that turn. This encourages turns with the same semantic operation to share expert sets while separating the experts used for different operations. At the token level, we regularize adjacent expert selections to be locally consistent. The control is deliberately selective: it reinforces the previous token's expert set only when two adjacent routes are already similar, preserving necessary routing changes at semantic and field boundaries. Because the introduced auxiliary router control can destabilize RL training, we further introduce an entropy-gated mechanism that disables the routing-control gradients when the policy entropy encounters high fluctuations. This framework preserves the token-wise top-$k$ router, and requires no architectural modification to existing MoE models. It can as well be combined with different policy-gradient algorithms such as PPO~\citep{schulman2017proximal}, GRPO~\citep{shao2024deepseekmath}, LOOP~\citep{chen2025reinforcement}, and GiGPO~\citep{feng2026group}. Through evaluations on AppWorld~\citep{trivedi2024appworld} and AutomationBench~\citep{shepard2026automationbench}, we show that the proposed routing control framework improves task performance when combined with various underlying RL algorithms, stabilizes MoE RL training, and increases inference efficiency. Overall, these results show that agentic trajectory structure provides a useful signal for organizing MoE expert selection during post-training.

We summarize the main contributions as follows:
\begin{itemize}
    \item We systematically study the connection between agentic behavior and MoE expert selection. We identify operation-grouped expert specialization and temporal routing consistency as key factors to achieve better task performance and inference efficiency. These routing structures are overlooked by standard RL training algorithms.

    \item We propose a hierarchical routing control framework that combines operation-aware turn-level specialization with outlier-tolerant token-level consistency. We also identify the stability risk of directly controlling MoE routers and introduce an entropy-gated control mechanism. The method supports the standard top-$k$ MoE router and is compatible with different RL algorithms.

    \item We conduct comprehensive experiments on AppWorld and AutomationBench. Our routing control framework consistently improves agentic RL with improved routing structures, over 10-point task performance improvement, long-term training stability, and better inference efficiency.
\end{itemize}

\begin{figure}[t]
    \centering
    \includegraphics[width=\linewidth]{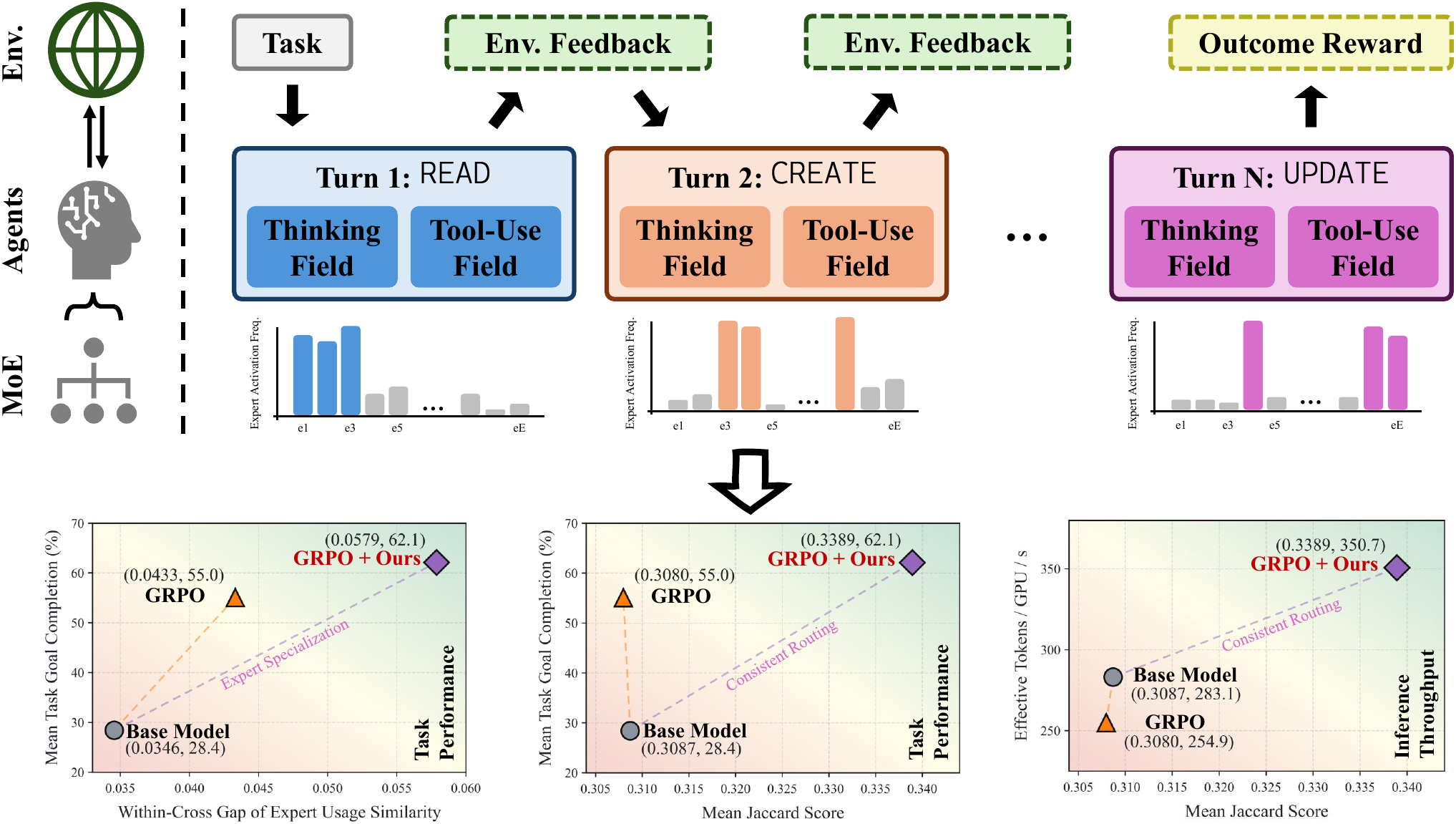}
    \caption{Overview of MoE routing control for agentic trajectories. \textbf{Top:} In each turn of interaction with the environment, agents generate \emph{thinking} and \emph{tool-use} fields, then receive \emph{environment feedback}. Our routing control framework aims to align expert selection with agentic operations (e.g., \texttt{READ}, \texttt{CREATE}, and \texttt{UPDATE}) in each turn, assigning distinct expert groups to different operations and encouraging similar expert selections within each turn. \textbf{Bottom:} With \texttt{Qwen3-30B-A3B} on AppWorld, our routing control framework (when used with standard RL) unlocks the performance limits and raises the inference throughput, by improving expert specialization (larger with-cross gap of expert usage similarity in Appendix~\ref{appendix:metric_specialization}) and consistency (higher Jaccard scores).}
    \label{fig:effect_overview}
\end{figure}

\section{Preliminary}
\subsection{Agentic Reinforcement Learning}
\label{sec:prelim_agentic_rl}

We define agentic trajectories to contain several rounds of reasoning and tool use, in the same spirit as ReAct~\citep{yao2023react}. For example, trajectories in AppWorld~\citep{trivedi2024appworld} and AutomationBench~\citep{shepard2026automationbench} involve reading information, creating an object, and updating an object in later turns. Concretely, at turn $i$, the policy $\pi_{\bm{\theta}}$ generates a \emph{thinking field} $y_i^{\text{think}}$ to perform textual reasoning for the task, followed by a \emph{tool-use field} $y_i^{\text{tool}}$ containing executable code or a tool call. The environment executes the tool use and returns feedback that guides subsequent turns. For a task $x$, this interaction produces an $N$-turn trajectory
\begin{equation}
    \tau=\left[y_1^{\text{think}}; y_1^{\text{tool}}; v_1;
    \ldots; y_N^{\text{think}}; y_N^{\text{tool}}; v_N\right],
    \label{eq:prelim_trajectory}
\end{equation}
where $v_i$ is the intermediate environment feedback.

Reinforcement learning (RL) uses an outcome reward $R(\tau)$ as the training signal. Specifically, let $\tau_i$ denote the task $x$ together with all generated fields and environment feedback before turn $i$. The simplest policy gradient for maximizing the expected outcome reward is
\begin{equation}
    \nabla_{\bm{\theta}} J_{\text{RL}}(\bm{\theta})
    =\mathbb{E}_{\tau\sim\pi_{\bm{\theta}}(\cdot | x)}\!\left[
    R(\tau) \cdot \sum_{i=1}^{N}\nabla_{\bm{\theta}}\log
    \pi_{\bm{\theta}}\!\left(y_i^{\text{think}},y_i^{\text{tool}}\mid \tau_i\right)
    \right].
    \label{eq:prelim_policy_gradient}
\end{equation}
Environment feedback supplies context but contributes no policy log-probability term~\citep{chen2025reinforcement}. Practical algorithms such as PPO~\citep{schulman2017proximal} or GRPO~\citep{shao2024deepseekmath} replace the outcome reward with an advantage score, requiring further value estimation.

\subsection{MoE Models: Efficiency and Expert Specialization}
\label{sec:prelim_moe}

Generating the thinking and tool-use fields requires a sequence of token-level computations. An MoE policy performs these computations with sparse activation: each MoE layer contains $M$ feed-forward experts, and a single-layer router selects only $k\ll M$ experts per token~\citep{shazeer2017}. Let $\mathcal{L}_{\text{MoE}}$ denote the set of MoE layers and $t$ a token position in the interaction context. At layer $l\in\mathcal{L}_{\text{MoE}}$, the router maps the hidden representation $\mathbf{h}_{l,t}$ to expert logits and then computes a softmax distribution over all experts:
\begin{equation}
    \mathbf{z}_{l,t}=W_l\mathbf{h}_{l,t},
    \qquad
    \mathbf{p}_{l,t}=\operatorname{softmax}(\mathbf{z}_{l,t}),
    \label{eq:prelim_routing}
\end{equation}
where $W_l$ is the router's weight matrix and $p_{l,t}[e]$ is the probability assigned to expert $e$. The router selects the top-$k$ experts, denoted by $\mathcal{S}_{l,t}$, and combines their output hidden states as a weighted sum using routing softmax probabilities. This sparse activation enables efficient parameter usage during inference~\citep{yang2025qwen3, qwen3.5}.

In agentic tasks, we study expert specialization~\citep{dai2024deepseekmoe, bo2026smoes, wang2026myth} through expert usage at both the token and turn levels. The expert-usage examples in Fig.~\ref{fig:effect_overview} illustrate our framework's intended outcome: different operations preferentially use distinct expert groups. Standard RL optimizes task completion without explicitly optimizing the routing strategy.

\section{Method}\label{sec:method}
In this section, we introduce the proposed hierarchical routing control framework, namely turn-level control (Section~\ref{sec:turn-level-control}) and token-level control (Section~\ref{sec:token-level-control}). We also discuss the MoE training stability issues induced by our framework and provide an effective solution in Section~\ref{sec:entropy-gated-control}.

Before introducing the routing control method, we analyze the default routing statistics of off-the-shelf MoE models in Fig.~\ref{fig:method_pre_analysis} and Fig.~\ref{fig:method_pre_analysis_automationbench}. At the turn level, we try different grouping criteria beyond operations and compare within-group and cross-group expert-usage similarity. This similarity is computed across the expert distributions of token pairs. The gap between within- and cross-group similarity has been used to monitor the expert specialization of MoE models~\citep{wang2026myth}. The results demonstrate that off-the-shelf MoE models naturally tend to use distinct expert groups for different operations. At the token level, we also observe that expert overlap for tokens within the same field and turn is relatively higher when compared to any cross-field or cross-turn scenario. Metrics used in the analysis are detailed in Appendix~\ref{appendix:definition}.

However, the default routing statistics are often perturbed during RL training. As shown in Section~\ref{sec:exp-routing}, RL injects noise into both token-level and turn-level routing statistics, which hinders task performance and inference efficiency. These issues motivate our explicit routing control framework to denoise expert selection.

\begin{figure}[t]
    \centering
    \includegraphics[width=\linewidth]{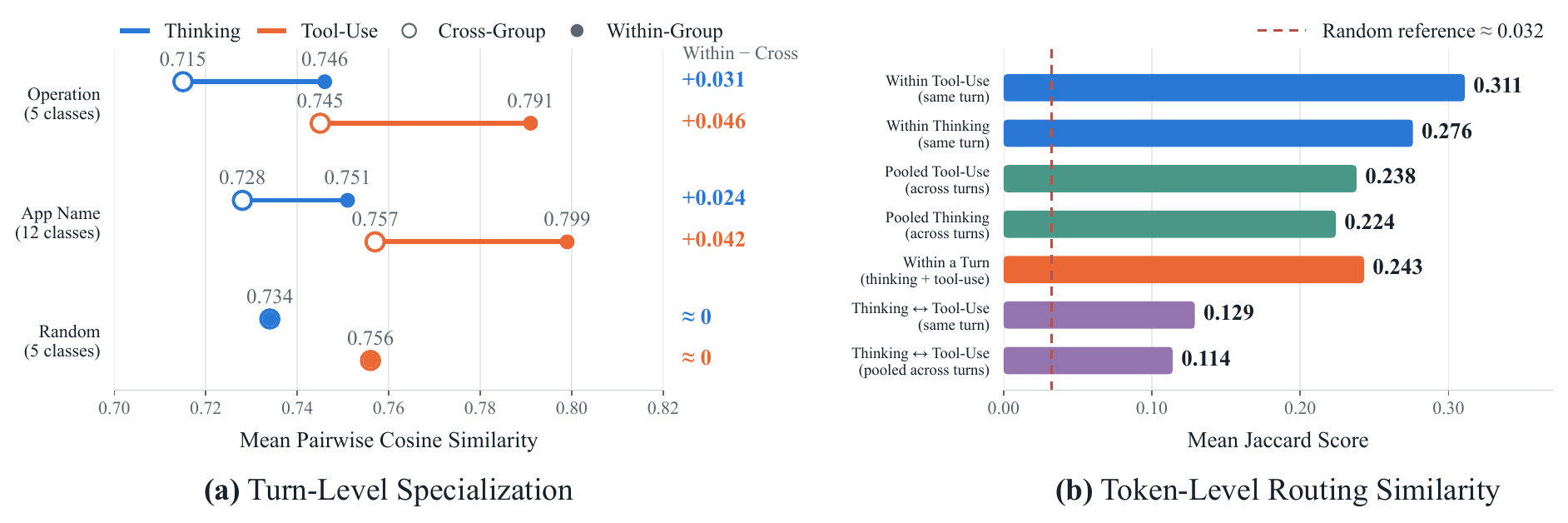}
    \caption{Default routing statistics of \texttt{Qwen3-30B-A3B} on AppWorld. \textbf{Left:} Pairwise cosine similarity of token-level expert distributions; the difference between within group and cross-group cosine similarity values indicates a relatively higher expert overlap between tokens that share the same operation label. \textbf{Right:} Jaccard score between token-level expert sets is relatively higher within the same field and turn indicating more overlapping expert routing patterns.}
    \label{fig:method_pre_analysis}
\end{figure}

\begin{figure}[t]
    \centering
    \includegraphics[width=\linewidth]{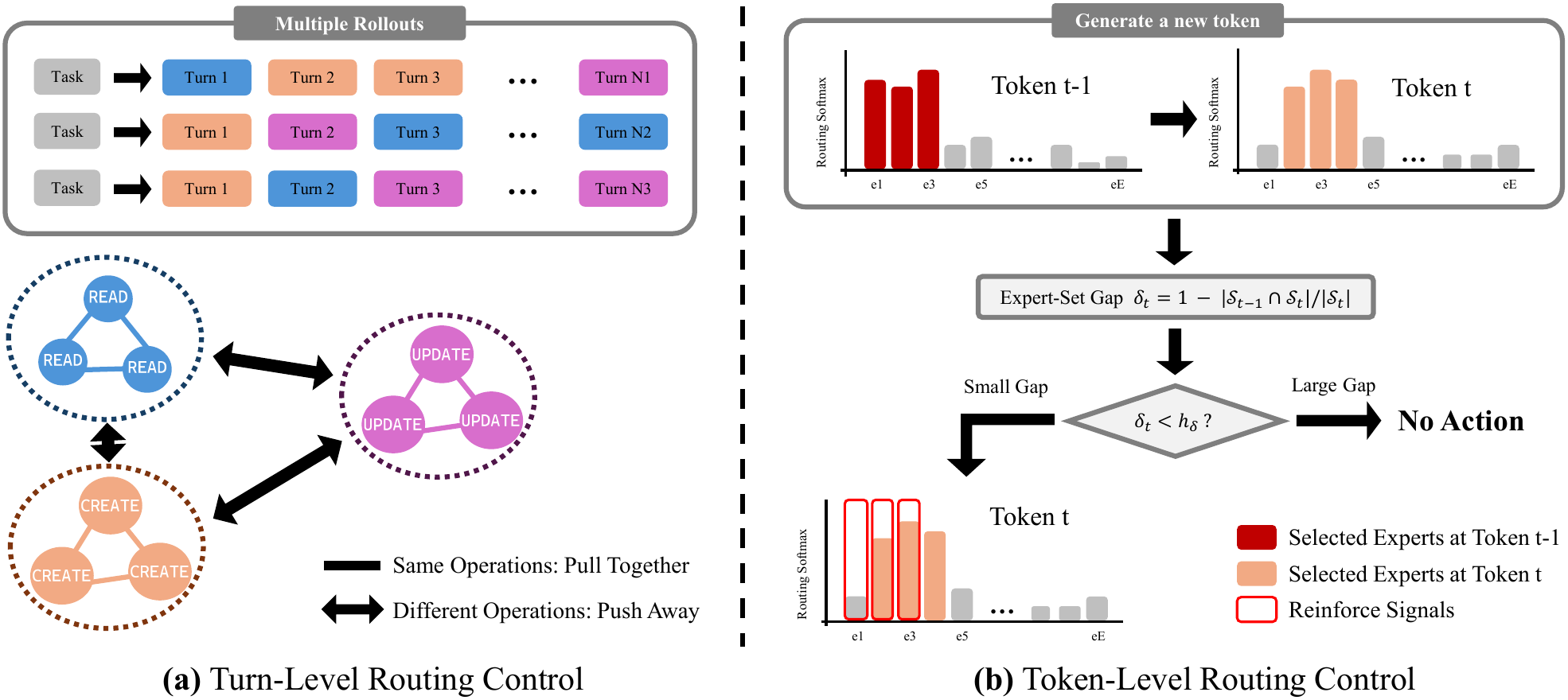}
    \caption{Overview of our hierarchical routing control framework. \textbf{Left:} 
    Turn-level control aggregates router distributions within each field and maximizes the mutual information between the routing distribution and the turn's operation labels. \textbf{Right:} Token-level control measures the top-$k$ expert-set gap between adjacent tokens and reinforces the preceding expert set only when this gap is small. The entropy gate is omitted here for clarity.}
    \label{fig:method_overview}
\end{figure}

\subsection{Turn-Level Operation-Aware Control}
\label{sec:turn-level-control}

Agentic trajectories consist of multiple interactive turns involving different semantic operations. For example, agents may capture information, create objects, or update existing states in different turns. The tools/APIs actually used may differ but can be broadly categorized into a few semantic groups. We show the concrete rules for grouping operations in Appendix~\ref{appendix:rubric}.

Inspired by \citet{bo2026smoes}, which uses mutual information (MI) between expert bins and modality to encourage expert specialization, we introduce MI-based regularization for turn-level routing control. For each MoE layer, we compute the following quantities separately for thinking and tool-use fields in a turn, omitting the layer and field-type indices for simplicity. Recalling Fig.~\ref{fig:method_pre_analysis}(b), the aforementioned routing structure only applies to a group of thinking/tool-use fields, and is not observable for cross-thinking-tool-use scenario.

Let $\mathcal{T}_i$ denote the token indices of the $i$-th turn in the trajectory. The turn-level expert selection distribution for $i$-th turn is defined as $\mathbf{q}_{i}(e) = \frac{1}{|\mathcal{T}_i|}\sum_{j\in\mathcal{T}_i}\mathbf{p}_j(e)$, where $\mathbf{p}_j(e)$ denotes the router expert selection probability for expert $e$ at token $j$. A turn may contain multiple operations. For each operation $o$, define $y_{io}=\mathbb{I}[o \in \mathcal{O}_i]$, where $\mathcal{O}_i\subset\mathcal{O}$ is the turn's operation-label set. We measure the association between expert identity $E$ and operation presence $Y_o$ through
\begin{equation}
    \mathcal{I}(E;Y_o) = \mathbf{KL}\left(P_o(E,Y_o) \middle\| P(E)\otimes P_o(Y_o)\right).
\end{equation}
Across $N$ eligible turns, we estimate the marginal and joint distributions of experts and operations:
\begin{equation}
    \begin{gathered}
    \hat{P}(e) = \frac{1}{N}\sum_{i=1}^{N} q_i(e),
    \qquad
    \hat{P}_o(y) = \frac{1}{N}\sum_{i=1}^{N}\mathbb{I}[y_{io}=y],
    \\
    \text{and}~~
    \hat{P}_a(e,y) = \frac{1}{N}\sum_{i=1}^{N} q_i(e)\mathbb{I}[y_{io}=y],~~ e \in \{e_1, e_2,...,e_M\},~~ y\in\{0,1\}.
    \end{gathered}
\end{equation}
Then, the stabilized turn-level MI objective is
\begin{equation}
    J_{\text{turn}} = \frac{1}{|\mathcal{O}|}\sum_{o\in\mathcal{O}}\sum_e\sum_{y\in\{0,1\}}\hat{P}_o(e,y) \cdot \operatorname{clip}\!\left(\log\frac{\hat{P}_o(e,y)+\epsilon}{\hat{P}(e)\hat{P}_o(y)+\epsilon},-h_c,h_c\right),
    \label{eq:J_turn}
\end{equation}
where $\mathcal{O}$ contains all possible operations. We use $\epsilon=10^{-12}$ for numerical stability and $h_c=20$ to bound the pointwise MI values used as control signals. Practically, we compute $J_{\text{turn}}$ over all trajectories in a batch, and require at least 8 turns to construct the expert-operation association. No MI signal is applied when there are not enough turns.

Maximizing $J_{\text{turn}}$ encourages turns with the same operation to favor similar experts and turns with different operations to develop distinct expert preferences. For each operation $o$, the log-ratio in Eq.~\ref{eq:J_turn} identifies experts that are over- or under-represented among turns with $Y_o=y_{io}$ relative to their overall usage. Turns containing the same operation receive a shared signal from its MI term, reinforcing experts associated with that operation and suppressing less associated experts. Turns without that operation receive complementary preferences, separating the two operation-conditioned mean distributions. Thus, the objective strengthens existing expert-operation associations and encourages consistency through shared expert preferences.

The relation to load balancing~\citep{shazeer2017} follows from the underlying MI decomposition $I(E;Y_o)=H(E)-H(E\mid Y_o)$. The marginal-entropy term favors diverse overall expert usage, while the conditional-entropy term favors more concentrated usage within each operation state. Standard load-balancing losses encourage approximately uniform expert utilization without using operation labels, while the MI objective additionally rewards operation-dependent usage. Therefore, $J_{\text{turn}}$ complements load balancing by promoting operation specialization, but does not guarantee uniform expert utilization. We also empirically compare load-balancing loss with the proposed turn-level routing control in Appendix~\ref{appendix:load_balance}.

\subsection{Token-Level Local Consistency Control}
\label{sec:token-level-control}

RL tends to inject noise into token-level expert selection, increasing expert switching between adjacent tokens (empirical evidence is shown in Section~\ref{sec:exp-routing}). This can disrupt local routing consistency within a thinking or tool-use field and introduce overhead during trajectory generation. \citet{yang2026phase} similarly motivates temporally consistent expert assignments for coherent agent behavior, using a switching penalty between consecutive environment steps. We comprehensively compare multiple variants of token-level routing control (comparison details in Appendix~\ref{appendix:variant}) and determine that, for widely used MoE architectures such as Qwen's MoE models, a simple adjacent-token loss with a gap threshold is the most reliable choice.

As illustrated in Fig.~\ref{fig:method_overview}(b), we reinforce the previous token's expert set only when the current selection is already similar to the previous token's. For adjacent token positions $t-1$ and $t$ within the same field, we measure the expert-set gap for each MoE layer as $\delta_{t} = 1 - |\mathcal{S}_{t-1} \cap \mathcal{S}_{t}|/|\mathcal{S}_{t}|$. Given a threshold $h_{\delta}\in[0,1]$, we encourage the current router to retain the previous token's expert selection by minimizing
\begin{equation}
    \ell_{t}^{\text{token}}
    = -\frac{\mathbb{I}[\delta_{t} < h_{\delta}]}{|\mathcal{S}_{t-1}|}
    \sum_{e\in S_{t-1}} \log p_{t}[e].
    \label{eq:token-loss}
\end{equation}
The previous token's router probabilities and the gate are treated as detached variables. Only the current token's probabilities are differentiable. The token-level objective $J_{\text{token}} = -\frac{1}{N}\sum_{i=1}^N\sum_{t\in\mathcal{T}_i}\ell_{t}^{\text{token}}$ averages this per-token loss over all within-field adjacent token pairs. Pairs crossing field or turn boundaries are excluded.

The expert-set gap threshold $h_{\delta}$ makes this control selective. A small gap is treated as a local perturbation of an existing route, whereas a large gap may reflect a meaningful semantic transition and receives no consistency penalty. Applying the penalty to every pair could suppress such transitions and impose overly rigid routing.

This control reduces local routing noise while allowing expert selections to change when needed, complementing the operation-level specialization in Section~\ref{sec:turn-level-control}. Empirically, it also improves rollout throughput and reduces rollout time (see Section~\ref{sec:exp-efficiency} for results and Appendix~\ref{appendix:efficiency_reason} for explanation).

\subsection{Entropy-Gated Control for Stable MoE RL}
\label{sec:entropy-gated-control}

The stability issues of MoE RL have been studied for single-turn tasks~\citep{ma2026stabilizing, zhang2026towards}. On multi-turn agentic tasks, our routing controls may also destabilize RL training. Small router score changes near the top-$k$ boundary can alter expert selection, and strong auxiliary updates may interfere with policy optimization. Empirically, training collapse is accompanied by a sharp fluctuation in policy entropy~\citep{wang2026on, li2026addressing, lochab2026uniform}, which motivates using it as a stability safeguard.

For a generated token position $t$, the policy entropy is defined as
\begin{equation}
    \mathcal{H}_t = -\sum_{v\in\mathbb{V}} \pi_{\bm{\theta}}(v|\tau_{<t}) \cdot \log\pi_{\bm{\theta}}(v|\tau_{<t}),
\end{equation}
where $\mathbb{V}$ is the vocabulary and $\tau_{<t}$ contains the generated tokens and environment feedback before position $t$. Let $\mathcal{H}$ denote its average over all thinking and tool-use tokens in the rollout batch at one RL step. We apply the routing control only when $H_{\text{low}} \leq \mathcal{H} \leq H_{\text{high}}$. Therefore, the entropy-gated objective is
\begin{equation}
    J = J_{\text{RL}} + g \cdot (\lambda_{\text{turn}}J_{\text{turn}} + \lambda_{\text{token}}J_{\text{token}}),
\end{equation}
where $g = \mathbb{I}[H_{\text{low}} \leq \mathcal{H} \leq H_{\text{high}}]$. When the gate closes, the routing-control gradients are detached while the standard RL updates continue. Routing control resumes when the monitored policy entropy $\mathcal{H}$ returns within the target range.

To empirically show these stability issues and the effect of entropy gating, we compare training curves using the same routing control and different entropy gates in Fig.~\ref{fig:entropy-gate}\footnote{We use $\lambda_{\text{turn}} = 2\times10^{-6}$ and $\lambda_{\text{token}} = 2\times10^{-7}$ here, as well as in the main experiments.}. Without entropy gating, policy entropy spikes and both evaluation splits collapse to zero around step 100. A threshold $H_{\text{high}}=0.4$ delays collapse to around step 140. With $H_{\text{high}}=0.2$, training remains stable, retaining task performance on both splits. This supports the effectiveness of entropy gating.

\begin{figure*}[t]
    \centering
    \includegraphics[width=\textwidth]{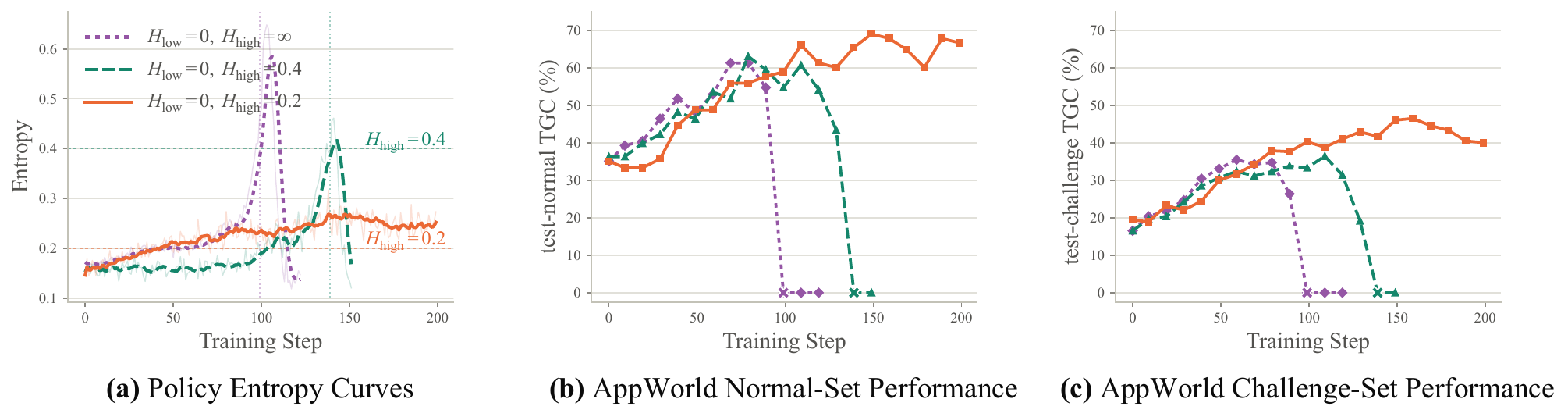}
    \caption{Effect of entropy gating on RL with routing control. Crosses mark the first evaluation at which both splits reach zero TGC. Routing control alone induces training collapse, and RL training remains stable for all 200 steps when an appropriate entropy gate (target range $[0, 0.2]$) is applied.}
    \label{fig:entropy-gate}
\end{figure*}

\section{Experiments}\label{sec:exp}
\subsection{Setup}

\paragraph{Benchmarks and metrics.}
We train on 90 AppWorld~\citep{trivedi2024appworld} tasks and 480 AutomationBench~\citep{shepard2026automationbench} public training tasks spanning six domains. For AutomationBench, we split the 600-sample public set into a training and a test subsets, each containing all six domains. AppWorld agents interact with applications through executable Python code. We report task goal completion (TGC), the percentage of fully completed tasks, and scenario goal completion (SGC), the percentage of scenarios with all constituent tasks completed, on \texttt{test-normal} and \texttt{test-challenge}. AutomationBench agents execute structured tool calls and receive partial-credit rewards based on the fraction of assertions passed; we report scores on withheld public subset, which are not used for training.

\paragraph{Models and training.}
We use the instruction-following versions of \texttt{Qwen3-30B-A3B-2507} for AppWorld and the more powerful \texttt{Qwen3.5-35B-A3B} for the harder AutomationBench, to guarantee initial success rates. The training and rollout generation of both models are colocated on 8 NVIDIA B200 GPUs. Within each baseline--control comparison, we keep the model, rewards, and RL hyperparameters fixed. We evaluate checkpoints periodically. The reported learning curves span 200 AppWorld steps and 100 AutomationBench steps.

\paragraph{Compared methods.}
We augment PPO~\citep{schulman2017proximal}, GRPO~\citep{shao2024deepseekmath}, LOOP~\citep{chen2025reinforcement}, and GiGPO~\citep{feng2026group} with our routing controls. Unless otherwise stated, ``+ Ours'' combines turn-level control, token-level control, and entropy gating (Section~\ref{sec:method}). We use GRPO on AppWorld for the routing and efficiency analysis.

\subsection{End-to-End Task Performance}
\label{sec:exp-performance}

Table~\ref{tab:appworld-main} shows that our routing control framework improves all RL algorithms across the 2 benchmarks. Specifically, we observe up to 12.9-point improvements in AppWorld~\citep{trivedi2024appworld}. For AutomationBench~\citep{shepard2026automationbench}, our routing control improves LOOP~\citep{chen2025reinforcement} by 11.34 points. The reported numbers are from the best-performance checkpoints on AppWorld \texttt{dev} set and AutomationBench public set respectively. We notice the PPO~\citep{schulman2017proximal} exhibits inconsistent improvement on AppWorld, and assume this to be the outcome of the small set of trajectories for a single prompt. Under PPO, each prompt has far fewer turns than under group-based RL, which makes the calculation of mutual information unstable.

\begin{table}[t]
    \caption{Agentic task performance (in percentages). \textbf{Bold} and \underline{underline} mark the highest and second-highest scores per column. Arrows show percentage-point changes from the corresponding baseline. Our routing control framework improve the agentic RL on both benchmarks, supporting diverse RL algorithms.}
    \label{tab:appworld-main}
    \centering
    \small
    \begin{tabular}{llllll}
        \toprule
        \multirow{4}{*}{\textbf{Method}}
        & \multicolumn{4}{c}{\textbf{AppWorld}}
        & \multicolumn{1}{c}{\textbf{AutomationBench}} \\
        & \multicolumn{4}{c}{\texttt{Qwen3-30B-A3B-2507}}
        & \multicolumn{1}{c}{\texttt{Qwen3.5-35B-A3B}} \\
        \cmidrule(lr){2-5}
        & \multicolumn{2}{c}{\textbf{test-normal}}
        & \multicolumn{2}{c}{\textbf{test-challenge}}
        & \\
        \cmidrule(lr){2-3}\cmidrule(lr){4-5}
        & TGC (\%) & SGC (\%) & TGC (\%) & SGC (\%) & \\
        \midrule
        Base Model
        & 35.1 & 14.3 & 18.7 & 7.2 & 22.04 \\
        \midrule
        PPO
        & 63.1 & 37.5 & 35.5 & 16.5 & 54.56 \\
        \quad + Ours
        & 60.7 \textcolor{red!70!black}{$(2.4\,\downarrow)$}
        & 42.9 \textcolor{green!50!black}{$(5.4\,\uparrow)$}
        & 38.6 \textcolor{green!50!black}{$(3.1\,\uparrow)$}
        & 20.9 \textcolor{green!50!black}{$(4.3\,\uparrow)$}
        & 57.39 \textcolor{green!50!black}{$(2.83\,\uparrow)$} \\
        \midrule
        GRPO
        & 69.6 & 46.4 & 40.5 & 22.3 & 63.01 \\
        \quad + Ours
        & \textbf{75.0} \textcolor{green!50!black}{$(5.4\,\uparrow)$}
        & \underline{55.4} \textcolor{green!50!black}{$(8.9\,\uparrow)$}
        & \underline{49.2} \textcolor{green!50!black}{$(8.6\,\uparrow)$}
        & \textbf{34.5} \textcolor{green!50!black}{$(12.2\,\uparrow)$}
        & 69.54 \textcolor{green!50!black}{$(6.53\,\uparrow)$} \\
        \midrule
        LOOP
        & 64.9 & 46.4 & 37.4 & 21.6 & 55.50 \\
        \quad + Ours
        & \underline{73.2} \textcolor{green!50!black}{$(8.3\,\uparrow)$}
        & \textbf{58.9} \textcolor{green!50!black}{$(12.5\,\uparrow)$}
        & \textbf{50.4} \textcolor{green!50!black}{$(12.9\,\uparrow)$}
        & \underline{30.9} \textcolor{green!50!black}{$(9.4\,\uparrow)$}
        & 66.84 \textcolor{green!50!black}{$(11.34\,\uparrow)$} \\
        \midrule
        GiGPO
        & 52.4 & 26.8 & 32.9 & 17.3 & \underline{70.01} \\
        \quad + Ours
        & 57.7 \textcolor{green!50!black}{$(5.4\,\uparrow)$}
        & 39.3 \textcolor{green!50!black}{$(12.5\,\uparrow)$}
        & 37.4 \textcolor{green!50!black}{$(4.6\,\uparrow)$}
        & 18.7 \textcolor{green!50!black}{$(1.4\,\uparrow)$}
        & \textbf{70.24} \textcolor{green!50!black}{$(0.23\,\uparrow)$} \\
        \bottomrule
    \end{tabular}
\end{table}

\subsection{Changes in Routing Behavior}
\label{sec:exp-routing}

Fig.~\ref{fig:routing-stats} examines the two controls separately. Under GRPO, the mean Jaccard similarity within each field stays relatively stable, while the token-level control increases the Jaccard score for 7.7\% (Figure~\ref{fig:routing-stats}(a)), indicating improved local consistency. This improved routing consistency demonstrates the effectiveness of the token-level control.

Turn-level control strengthens the expert-operation alignment in later stages. We use the mutual information (MI) between turn-level expert distributions and operation labels (the same as the turn-level control objective) to quantify this dependence~\citep{bo2026smoes}. Late-stage MI increases 23.5\% (Figure~\ref{fig:routing-stats}(b)). Inspired by \citet{wang2026myth}, we also quantify the expert specialization via the gap between within- and cross-group expert-usage similarity, as defined in Appendix~\ref{appendix:metric_specialization}. This diagnostic measures whether tokens tend to use overlapped experts more when they are under the same operation label than different labels (Fig.~\ref{fig:routing-stats}(c)). The results support greater operation separability.

\begin{figure*}[t]
    \centering
    \includegraphics[width=\textwidth]{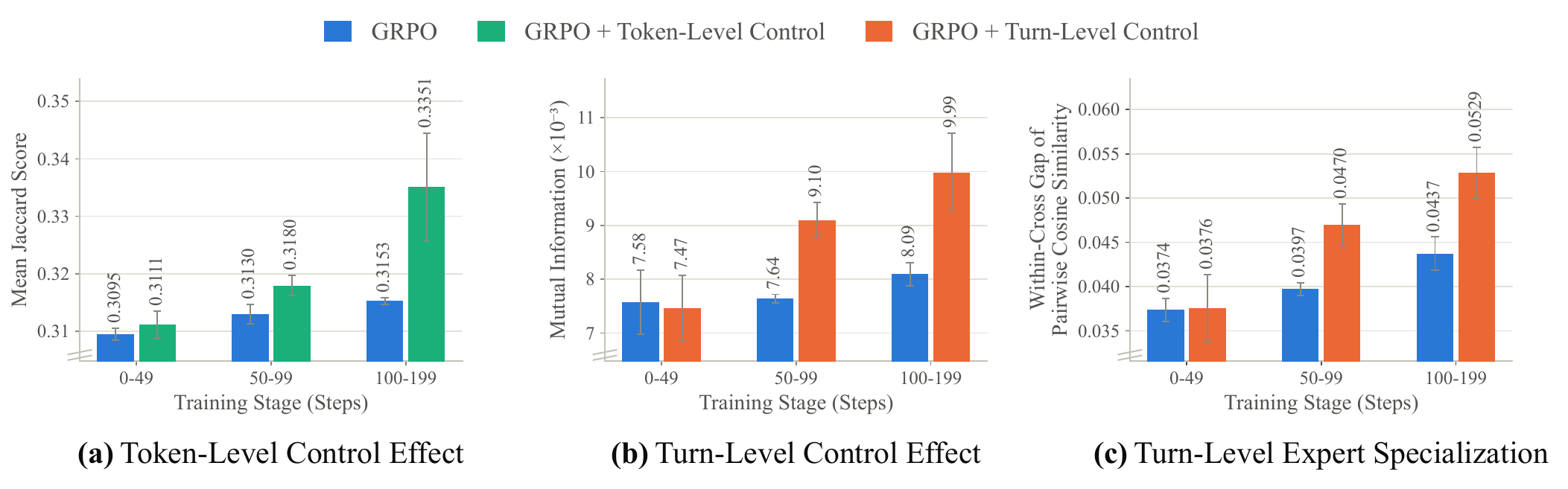}
    \caption{Routing statistics during AppWorld GRPO training. (a) Token-level control improve routing similarity. The numbers are averaged over all token pairs within fields. (b) Turn-level control increases the mutual information between expert selection and operation labels. (c) Turn-level control makes expert selections more separable by the operation labels. The reported numbers are averaged over all of the evaluation trajectories collected from the specified training step intervals. The error bars are the standard deviation of all data points.}
    \label{fig:routing-stats}
\end{figure*}

\subsection{Inference Efficiency}
\label{sec:exp-efficiency}

Fig.~\ref{fig:throughput_bars_extended} reports rollout wall time and effective generated tokens per GPU second. Each control improves both metrics across all three training stages. In steps 100--200, the combined method increases throughput from 224 to 316 tokens/GPU/s (+41.1\%) and reduces wall time from 224 to 174 seconds (a 22.3\% reduction). The combination achieves the highest late-stage throughput, whereas token-level control alone gives the shortest wall time (133 seconds). These end-to-end measurements demonstrate that improving the routing consistency has benefit on inference efficiency. We also have a detailed discussion about this efficiency benefit in Appendix~\ref{appendix:efficiency_reason}.

\begin{figure}[t]
    \centering
    \includegraphics[width=0.9\linewidth]{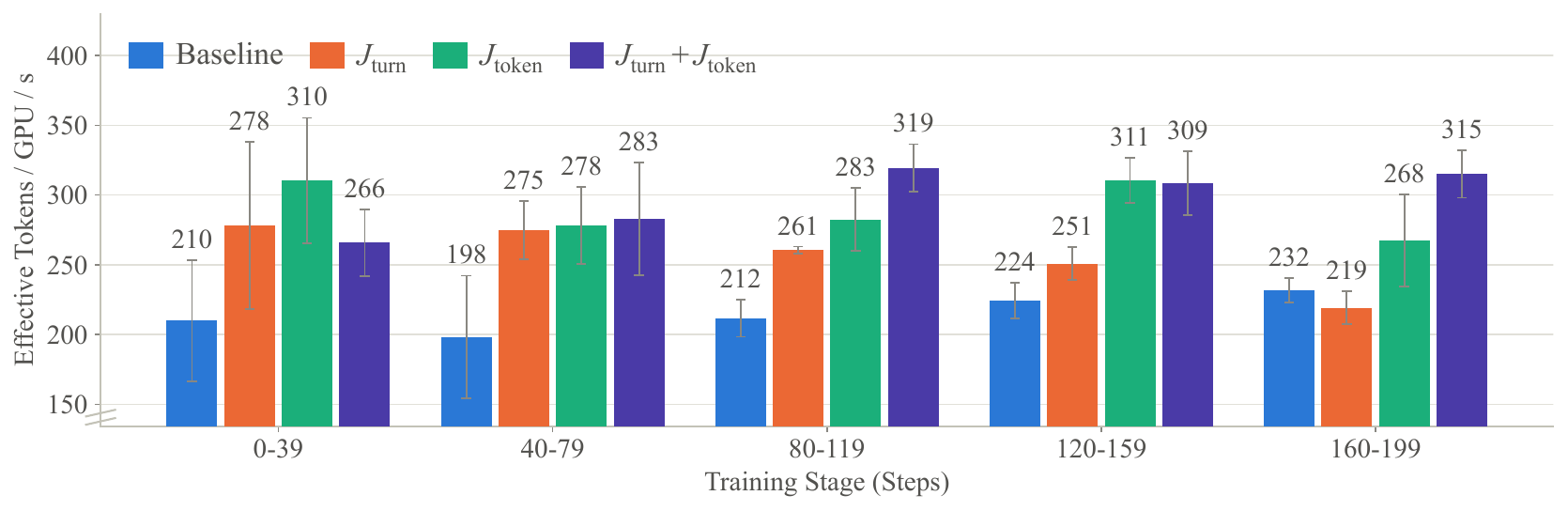}
    \caption{AppWorld rollout efficiency comparison under GRPO with turn-level control, token-level control, or both. Both controls improve inference throughput. The token-level control contributes the most to inference efficiency at the early training stage. The reported numbers are averaged over all of the evaluation trajectories collected from the specified training step intervals. The error bars are the standard deviations of all data points.}
    \label{fig:throughput_bars_extended}
\end{figure}

\section{Related Works}
\subsection{RL for Interactive LLM Agents}
\label{sec:related_agentic_rl}

Reinforcement learning enables LLM agents to improve through repeated interaction with stateful environments. LOOP~\citep{chen2025reinforcement} demonstrates this approach on long-horizon interactive tasks using a critic-free policy optimization objective. However, terminal outcome rewards provide limited guidance about the contribution of individual turns. GiGPO~\citep{feng2026group} derives step-level advantages by grouping actions at shared environment states, while SALT~\citep{li2026salt} constructs trajectory graphs to refine advantage assignment from outcome rewards. RLVMR~\citep{zhang2025rlvmr} provides additional process supervision through verifiable rewards for intermediate meta-reasoning behaviors.

Sparse rewards also make successful trajectories difficult to discover. SGE~\citep{szot2026expanding} guides exploration with high-level language strategies, while ADAPO~\citep{petrenkoentropy} adaptively adjusts policy clipping to preserve entropy and output diversity. Other studies also address the cost of environment interaction. Simia-RL~\citep{li2025simulating} uses reasoning models to simulate environment feedback and rewards, reducing the engineering required to construct executable training environments. SAGE~\citep{wang2026reinforcement} learns to generate and reuse a skill library, reducing interaction steps and generated tokens when solving related tasks. These approaches improve the learning signal, exploration, or interaction procedure. Our work studies how agentic RL changes expert selection within MoE models and uses the structure of turns to regularize their routing. This control complements the underlying policy-gradient estimator.

\subsection{MoE Routing Control}
\label{sec:related_moe_routing}

Recent work studies routing as a source of both instability and adaptability during MoE post-training. Rollout Routing Replay (R3)~\citep{ma2026stabilizing} stabilizes RL by replaying rollout-time expert assignments during training, aligning the discrete computation paths used for sampling and optimization. MoE-GRPO~\citep{ko2026moe} instead treats expert selection as a stochastic policy and uses reward-based updates to explore expert combinations in vision-language models. These methods address consistency between training and inference or reward-driven routing optimization, without explicitly organizing expert usage around the operations performed by an interactive agent.

Temporal consistency provides another signal for routing control. \citet{shen2026temporally} introduce temporally extended MoE models that retain expert selections across multiple tokens and quantify routing persistence through expert switch rates. ReMoE~\citep{zhu2026remoe} fine-tunes native routers with temporal-locality losses to improve expert reuse under memory-constrained inference. Our token-level control shares the goal of local consistency, but selectively reinforces the preceding expert set only when adjacent selections are already similar within a thinking or tool-use field. This leaves larger routing changes unpenalized by the token-level loss and preserves the native top-$k$ selection rule.

PA-MoE~\citep{yang2026phase} connects routing specialization with agentic RL by learning phase-dependent expert selection and regularizing switching across environment steps. Its reported implementation uses added LoRA experts across transformer layers. SMoES~\citep{bo2026smoes} also promotes structured specialization, maximizing mutual information between soft modality scores and expert bins in vision-language models. Our turn-level objective instead uses operation labels extracted from generated tool-use fields to regularize the native expert distributions. It strengthens existing expert-operation associations without assigning operations to expert groups in advance. Our routing control framework combines this operation-level signal with selective token-level consistency to optimize the routing strategy during agentic RL.

\section{Conclusion}
We study how agentic post-training interacts with expert selection in mixture-of-experts (MoE) agents. We observe that agentic trajectories naturally contain structured expert usage patterns related to the operation labels of each turn, and that ignoring such structure induces suboptimal task performance and efficiency. To address this problem, we introduce a hierarchical routing control framework that combines operation-aware turn-level control with selective token-level routing control. We further use policy entropy to gate the auxiliary routing gradients, preventing the instability caused by routing control. Experiments on AppWorld and AutomationBench show that the proposed framework improves the task performance of multiple RL algorithms, induces the intended routing structure, stabilizes long-term training, and also improves inference efficiency. Overall, our results demonstrate that the structure of agentic trajectories provides a useful signal for organizing MoE usage during post-training.

\section*{AI Use Statement}
We used generative AI tools to improve the clarity and readability of the manuscript, assist in preparing tables and figures, and support code development. The authors take full responsibility for the final content of this work, including its text, tables, figures, and scientific claims.

\bibliography{main}
\bibliographystyle{plainnat}

\small
\newpage

\tableofcontents
\newpage

\appendix

\section{Metric Definitions}\label{appendix:definition}
In this section, we define several evaluation metrics used in the empirical analysis. These metrics provide a foundation for understanding MoE model architectures and routing strategies. We show the specialization and consistency measurement along the RL training in Fig.~\ref{fig:routing_stats_appendix}.

\begin{figure}[!h]
    \centering
    \includegraphics[width=0.7\linewidth]{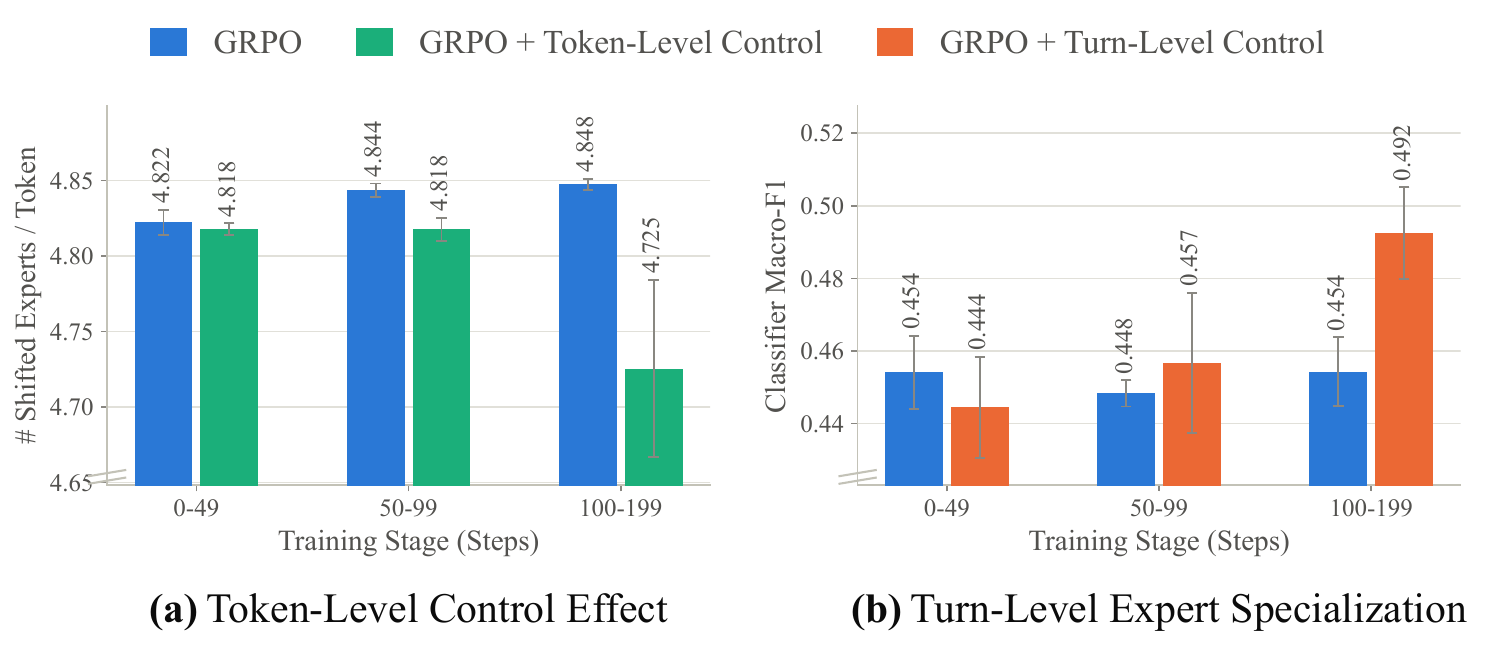}
    \caption{Routing behavior changes during training. Measured by the expert switches and expert-operation classification, our routing control framework improves the routing consistency and expert specialization.}
    \label{fig:routing_stats_appendix}
\end{figure}

Additionally, we show the default routing statistics for \texttt{Qwen3.5-35B-A3B} in Fig.~\ref{fig:method_pre_analysis_automationbench} using our defined metrics in this section. The routing statistics for \texttt{Qwen3.5-35B-A3B} is similar to that of \texttt{Qwen3-30B-A3B}. Off-the-shelf MoE models naturally exhibit expert specialization and routing consistency.

\begin{figure}[!h]
    \centering
    \includegraphics[width=\linewidth]{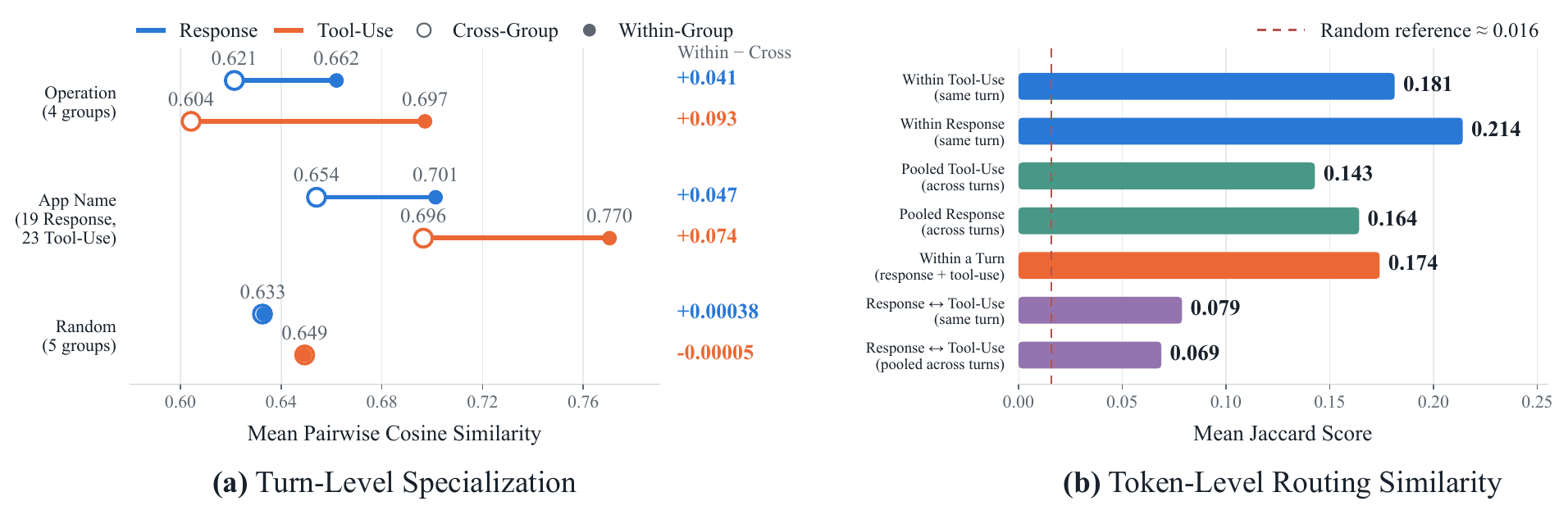}
    \caption{Default routing statistics of \texttt{Qwen3.5-35B-A3B} on Automation. \textbf{Left:} Pairwise cosine similarity of token-level expert distributions; the difference between within group and cross-group cosine similarity values indicates a relatively higher expert overlap between tokens that share the same operation label. \textbf{Right:} Jaccard score between token-level expert sets is relatively higher within the same field and turn indicating more overlapping expert routing patterns.}
    \label{fig:method_pre_analysis_automationbench}
\end{figure}

\subsection{Measuring Expert Specialization}
\label{appendix:metric_specialization}

\subsubsection{Expert-Operation Classification}
To directly show the association between expert selection and turn operations, we train a nearest-centroid classifier to predict turn labels based on input expert indices. We split the expert-operation pairs in the evaluation set into $K$ subsets and use $K$-fold cross-validation to ensure fair scores. Then, we pool the held-out predictions from all $K$ folds and compute
\begin{equation}
    \mathbf{F1}
    = \frac{1}{|\mathcal{O}|}\sum_{o\in\mathcal{O}}
    \frac{2\mathrm{TP}_o}
    {2\mathrm{TP}_o+\mathrm{FP}_o+\mathrm{FN}_o},
\end{equation}
where $\mathcal{O}$ is the set of observed operation classes in the evaluation group, and the counts are computed from the pooled predictions. This metric measures operation discriminability within the evaluation collection.

\subsubsection{Similarity Gap}
To measure turn-level expert specialization in Fig.~\ref{fig:method_pre_analysis}(a), we examine whether turns with the same operation label have more similar routing distributions than turns with different labels. Following \citet{wang2026myth}, we compare expert-usage frequency vectors using cosine similarity. The gap between within- and across-group similarities measures routing specialization with respect to these labels.

We first summarize token-level expert selections into one routing distribution per field in each turn, then group fields by their turn labels. Within-group similarity averages over distinct fields in the same group, whereas across-group similarity averages over fields from two different groups. Both quantities can be computed exactly from sums of normalized routing distributions.

\paragraph{Turn-level routing distributions.}
For a fixed MoE layer and field type (thinking or tool use), let $f$ index a nonempty field instance containing $T_f$ tokens. Let $\bm{m}_t^f \in \{0,1\}^{E}$ indicate the $k$ experts selected at token position $t$, where $E$ is the number of experts and $\sum_{e=1}^{E}\bm{m}_t^f[e]=k$. The routing distribution of field $f$ is
\begin{equation}
    \bm{p}^f = \frac{1}{T_f k}\sum_{t=1}^{T_f}\bm{m}_t^f,
    \qquad \sum_{e=1}^{E}\bm{p}^f[e]=1.
    \label{eq:specialization-routing-distribution}
\end{equation}
We compute the following quantities separately for each layer and field type, omitting these indices for clarity.

\paragraph{Within- and across-group cosine similarity.}
Let $\{\mathcal{G}_1,\ldots,\mathcal{G}_M\}$ be the nonempty turn-label groups, where $\mathcal{G}_m$ contains fields whose turns carry label $m$, and write $n_m=|\mathcal{G}_m|$. Define the normalized field vector and its group sum as
\begin{equation}
    \bm{u}^f = \frac{\bm{p}^f}{\|\bm{p}^f\|_2},
    \qquad
    \bm{s}(\mathcal{G}_m)=\sum_{f\in\mathcal{G}_m}\bm{u}^f.
    \label{eq:specialization-group-sum}
\end{equation}
For $n_m\geq 2$, the within-group similarity excludes self-pairs:
\begin{equation}
        \mathbf{WS}(\mathcal{G}_m)
        = \frac{1}{n_m(n_m-1)}
           \sum_{\substack{f,g\in\mathcal{G}_m\\f\neq g}}
           (\bm{u}^f)^\top\bm{u}^g
        = \frac{\|\bm{s}(\mathcal{G}_m)\|_2^2-n_m}
                 {n_m(n_m-1)}.
    \label{eq:specialization-within}
\end{equation}
For $m\neq n$, the across-group similarity is
\begin{equation}
        \mathbf{CS}(\mathcal{G}_m,\mathcal{G}_n)
        = \frac{1}{n_m n_n}
           \sum_{f\in\mathcal{G}_m}\sum_{g\in\mathcal{G}_n}
           (\bm{u}^f)^\top\bm{u}^g
        = \frac{\bm{s}(\mathcal{G}_m)^\top\bm{s}(\mathcal{G}_n)}
                 {n_m n_n}.
    \label{eq:specialization-across}
\end{equation}

\begin{lemma}[Exact computation from group sums]
\label{lem:specialization-complexity}
Given the $E$-dimensional field routing distributions, all within-group similarities for groups of size at least two and all across-group similarities can be computed exactly in
\begin{equation}
    O\!\left(E\sum_{m=1}^{M}n_m+EM^2\right)
\end{equation}
time by computing and caching each group sum once. For equally sized groups with $n_m=n$, this reduces the time complexity from $O(EM^2n^2)$ for explicit pairwise computation to $O(E(Mn+M^2))$.
\end{lemma}

\begin{proof}
Since $\|\bm{u}^f\|_2=1$, expanding the squared group-sum norm gives
\begin{equation*}
        \|\bm{s}(\mathcal{G}_m)\|_2^2
        = \sum_{f,g\in\mathcal{G}_m}(\bm{u}^f)^\top\bm{u}^g
        = n_m+
           \sum_{\substack{f,g\in\mathcal{G}_m\\f\neq g}}
           (\bm{u}^f)^\top\bm{u}^g.
\end{equation*}
Subtracting the $n_m$ self-pairs and dividing by $n_m(n_m-1)$ proves Eq.~\ref{eq:specialization-within}. Similarly, bilinearity gives
\begin{equation*}
    \bm{s}(\mathcal{G}_m)^\top\bm{s}(\mathcal{G}_n)
    = \sum_{f\in\mathcal{G}_m}\sum_{g\in\mathcal{G}_n}
      (\bm{u}^f)^\top\bm{u}^g,
\end{equation*}
which proves Eq.~\ref{eq:specialization-across} after division by $n_m n_n$.

Normalizing the field distributions and constructing all cached group sums costs $O(E\sum_m n_m)$. Computing their $M$ squared norms and $M(M-1)/2$ pairwise inner products costs $O(EM^2)$. In contrast, explicit comparison of all field pairs across $M$ groups of size $n$ costs $O(EM^2n^2)$.
\end{proof}

\subsection{Measuring Temporal Routing Consistency}
\label{appendix:metric_consistency}

\subsubsection{Jaccard Similarity}
Let $\mathcal{S}_{l,t}$ be the set of experts selected by MoE layer $l$ for token $t$. The token-level Jaccard similarity between token $t_1$ and token $t_2$ is defined as
\begin{equation}
    \mathbf{J}(t_1, t_2) = \frac{|\mathcal{S}_{l,t_1} \cap \mathcal{S}_{l,t_2}|}{|\mathcal{S}_{l,t_1} \cup \mathcal{S}_{l,t_2}|}.
\end{equation}
Then, for an agentic trajectory containing $N$ turns, $\tau = \left[y_1^{\text{think}}; y_1^{\text{tool}}; v_1; \ldots; y_N^{\text{think}}; y_N^{\text{tool}}; v_N\right]$, the within-field Jaccard similarity is
\begin{equation}
    \mathbf{J}_{\text{field}}(\tau) = \frac{1}{|\mathcal{L}_{\text{MoE}}| \cdot N}\sum_{l\in\mathcal{L}_{\text{MoE}}}\sum_{i=1}^N \left[\frac{1}{2\binom{|y_i^{\text{think}}|}{2}}\sum_{t_a,t_b\in y_i^{\text{think}}}\mathbf{J}(t_a,t_b) + \frac{1}{2\binom{|y_i^{\text{tool}}|}{2}}\sum_{t_a,t_b\in y_i^{\text{tool}}}\mathbf{J}(t_a,t_b)\right],
\end{equation}
which is an average over all within-field token pairs. The comparative metrics in Fig.~\ref{fig:method_pre_analysis}(b) are defined similarly, with different ways of constructing token pairs.

Additionally, suppose each token independently selects a uniformly random set of $k$ experts from $E$ available experts. The expected intersection and union sizes are $k^2/E$ and $2k-k^2/E$, respectively. Their ratio gives the approximate reference
\begin{equation}
    \mathbf{J}_{\text{ref}}
    = \frac{k^2/E}{2k-k^2/E}
    = \frac{k}{2E-k}.
    \label{eq:routing-random-reference}
\end{equation}
For \texttt{Qwen3-30B-A3B}, where $E=128$ and $k=8$, this gives $\mathbf{J}_{\text{ref}} \approx 0.032$, matching the dashed reference in Fig.~\ref{fig:method_pre_analysis}(b).

\subsubsection{Expert Switches}
Prior work also uses the number of expert switches per token to quantify routing consistency~\citep{yang2026phase}. We define the expert switch count over all within-field token pairs:
\begin{equation}
    \mathbf{ES}(\tau) = \frac{1}{|\mathcal{L}_{\text{MoE}}| \cdot N}\sum_{l\in\mathcal{L}_{\text{MoE}}}\sum_{i=1}^N \Biggl[\frac{1}{|y_i^{\text{think}}|}\sum_{t\in y_i^{\text{think}}} \left(k - |\mathcal{S}_{l,t} \cap \mathcal{S}_{l,t+1}|\right) + \frac{1}{|y_i^{\text{tool}}|}\sum_{t \in y_i^{\text{tool}}} \left(k - |\mathcal{S}_{l,t} \cap \mathcal{S}_{l,t+1}|\right)\Biggr],
\end{equation}
consistent with our hierarchical framework, in which routing consistency is measured only within each turn.

\subsection{Measuring Expert Load Balance}
\label{appendix:metric_load_balance}

Following \citet{shazeer2017}, we use the coefficient of variation (CV) to quantify expert-load imbalance. For a trajectory $\tau$, let $\mathcal{T}_{\tau}$ contain its generated thinking and tool-use tokens, and let $\mathcal{S}_{\ell,t}$ denote the top-$k$ experts selected at MoE layer $\ell$ for token $t$. The assignment count of expert $e$ is
\begin{equation}
    c_{l,e}(\tau)
    = \sum_{t\in\mathcal{T}_{\tau}}
      \mathbb{I}\!\left[e\in\mathcal{S}_{l,t}\right].
\end{equation}
We define the load CV as
\begin{equation}
    \mathbf{CV}_{l}(\tau)
    = \frac{\operatorname{std}_{e}\!\left(c_{l,e}(\tau)\right)}
           {\operatorname{mean}_{e}\!\left(c_{l,e}(\tau)\right)},
    \label{eq:expert_load_cv}
\end{equation}
where both statistics are computed over all experts in the layer, including those with zero assignments. We compute the unsquared CV separately for each trajectory and MoE layer, then average these values. Lower CV indicates more balanced expert usage, with zero corresponding to equal assignment counts across experts.

\section{Variant Routing Controls}\label{appendix:variant}
In this section, we introduce the routing control variants explored in our early research. Although these methods were less promising overall than our final framework in Section~\ref{sec:method}, they help explain which routing signals are useful for agentic RL. Table~\ref{tab:routing-variants} compares their task performance on AppWorld~\citep{trivedi2024appworld}, with the best-performing hyperparameters selected through a grid search ranging from $10^{-7}$ to $10^{-5}$. We do not apply the entropy gate and main experiments' hyperparameters for our routing controls.

\begin{table}[!h]
    \caption{Comparison of routing control variants on AppWorld, measured by task goal completion (TGC). The highest and second-highest scores in each column are bolded and underlined, respectively.}
    \label{tab:routing-variants}
    \centering
    \setlength{\tabcolsep}{5pt}
    \begin{tabular}{lcc}
        \toprule
        \multirow{2}{*}{\textbf{Method}}
        & \textbf{test-normal} & \textbf{test-challenge} \\
        \cmidrule(lr){2-2}\cmidrule(lr){3-3}
        & TGC (\%) & TGC (\%) \\
        \midrule
        Base Model & 34.52 & 19.66 \\
        \midrule
        GRPO & 66.70 & 38.80 \\
        ~~~~+ \texttt{routing\_reg} & 70.24 & 43.17 \\
        ~~~~+ \texttt{switch\_reg}~\citep{yang2026phase} & 72.02 & \underline{46.28} \\
        ~~~~+ \texttt{outlier\_reg} & 0.0 (crashed) & 0.0 (crashed) \\
        ~~~~+ \texttt{routing\_reward} & 73.21 & 41.25 \\
        ~~~~+ \texttt{routing\_clip} & \underline{73.81} & 45.08 \\
        ~~~~+ \texttt{load\_bal}~\citep{shazeer2017} & \textbf{74.40} & 45.56 \\
        \midrule
        ~~~~+ Ours & 72.62 & \textbf{51.80} \\
        \bottomrule
    \end{tabular}
\end{table}

\paragraph{Shared notation.}
We retain the notation from Section~\ref{sec:prelim_moe}: $M$ experts, top-$k$ selections $\mathcal S_{l,t}$, router logits $\mathbf{z}_{l,t}$, probabilities $\mathbf{p}_{l,t}=\operatorname{softmax}(\mathbf{z}_{l,t})$, and MoE layers $l\in\mathcal{L}_{\text{MoE}}$. Let $\mathcal T$ contain the token positions of one generated thinking or tool-use field, omitting its turn and field-type indices. Extending the assignment counts in Appendix~\ref{appendix:metric_load_balance}, define
\begin{equation}
    c_{l,e}(\mathcal T)
    = \sum_{t\in\mathcal T}\mathbb I[e\in\mathcal S_{l,t}],
    \qquad
    \mathcal C_l^{(K)}(\mathcal T)
    = \operatorname{TopK}_{e}\bigl(c_{l,e}(\mathcal T),K\bigr).
    \label{eq:variant-consensus}
\end{equation}
Here, $\operatorname{TopK}$ returns expert indices, and $\mathcal C_l(\mathcal T)=\mathcal C_l^{(k)}(\mathcal T)$ is the majority-voted consensus. Thus, ``majority-voted'' refers to the most frequently selected experts, without requiring selection by more than half the tokens. We compute consensus separately for each field in a turn. All adjacent-token expressions below require at least two tokens and exclude field and turn boundaries.

\subsection{Routing Regularization: Field Consensus}
\label{appendix:routing-reg}

\texttt{routing\_reg} encourages every token in a field to favor its majority-voted experts. We monitor agreement with the uniform distribution over $\mathcal C_l(\mathcal T)$ using
\begin{equation}
    \mathcal L^{\mathrm{cons}}_{l,\mathcal T}
    = -\frac{1}{k|\mathcal T|}
    \sum_{t\in\mathcal T}\sum_{e\in\mathcal C_l(\mathcal T)}
    \log p_{l,t}[e].
    \label{eq:variant-consensus-monitor}
\end{equation}
The above field-wide target can be too restrictive: different tokens within one field may require different experts, which explains its suboptimal performance.

\subsection{Switch Regularization: Adjacent-Distribution Agreement}
\label{appendix:switch-reg}

We adapt the temporal-consistency penalty of PA-MoE~\citep{yang2026phase}, which penalizes expert changes between adjacent tokens at each native MoE layer.

\texttt{switch\_reg} directly extends the expert switch count into a differentiable objective. We use the full router distributions:
\begin{equation}
    \ell^{\mathrm{switch}}_{l,t}
    = 1-\langle p_{l,t-1},p_{l,t}\rangle,
    \qquad
    \mathcal L^{\mathrm{switch}}_{l,\mathcal T}
    = \frac{1}{|\mathcal T|-1}
    \sum_{t:\,t-1,t\in\mathcal T}\ell^{\mathrm{switch}}_{l,t}.
    \label{eq:variant-switch-soft}
\end{equation}
The inner product is the probability that independent expert draws agree. Its complement therefore favors shared, concentrated preferences. Writing $d_{l,t}=\langle p_{l,t-1},p_{l,t}\rangle$, the pair derivatives are
\begin{equation}
    \begin{aligned}
        \nabla_{z_{l,t}}\ell^{\mathrm{switch}}_{l,t}
        &=p_{l,t}\odot\bigl(d_{l,t}\mathbf 1-p_{l,t-1}\bigr),\\
        \nabla_{z_{l,t-1}}\ell^{\mathrm{switch}}_{l,t}
        &=p_{l,t-1}\odot\bigl(d_{l,t}\mathbf 1-p_{l,t}\bigr),
    \end{aligned}
    \label{eq:variant-switch-derivatives}
\end{equation}
where $\mathbf 1$ is the all-ones vector and $\odot$ denotes elementwise multiplication. Both endpoints receive gradients; the neighboring probabilities are not detached from the overall surrogate.

\texttt{switch\_reg} can destabilize RL training by changing the full routing distribution. In particular, even identical adjacent distributions incur $1-\|p\|_2^2$, so the surrogate also favors concentration and can substantially alter router entropy. This explanation for the instability is also consistent with the sensitivity discussed in Appendix~\ref{appendix:router_entropy}. Router entropy here is different from the next-token policy entropy used for gating in Section~\ref{sec:entropy-gated-control}. The evaluated \texttt{switch\_reg} run improves AppWorld test-challenge TGC to $46.28\%$, but remains below our final framework.

\subsection{Outlier Regularization}
\label{appendix:outlier-reg}

\texttt{outlier\_reg} suppresses selected experts outside a field's consensus expert set. For a configurable consensus size $K$, define
\begin{equation}
    \mathcal O_{l,t}
    = \mathcal S_{l,t}\setminus\mathcal C_l^{(K)}(\mathcal T),
    \qquad
    m^{\mathrm{outlier}}_{l,t}
    = \sum_{e\in\mathcal O_{l,t}}p_{l,t}[e].
    \label{eq:variant-outliers}
\end{equation}
The design aims to reduce this selected outlier mass by applying probability-weighted downward signals to outlier logits. Only selected experts outside the target expert set are penalized; the remaining experts receive no direct penalty. We explored $K=8$ and $K=16$.

In our experiments, suppressing these selections perturbed routing too strongly and caused severe performance degradation; the run in Table~\ref{tab:routing-variants} crashed. This result motivated reinforcing existing local agreement instead of directly penalizing departures from a field-wide consensus.

\subsection{Routing Reward and Routing Clipping}
\label{appendix:routing-reward-clip}

These two variants incorporate routing consistency into the policy gradient objective (Eq.~\ref{eq:prelim_routing}). From recorded expert selections, we compute the detached consensus-overlap score
\begin{equation}
    \kappa_t
    = \frac{1}{|\mathcal L_{\mathrm{MoE}}|}
    \sum_{l\in\mathcal L_{\mathrm{MoE}}}
    \frac{|\mathcal S_{l,t}\cap\mathcal C_l(\mathcal T)|}{k}.
    \label{eq:variant-routing-score}
\end{equation}
We set $\kappa_t=1$ outside the configured fields. Let $\widehat A_t$ be the existing advantage and $r_t(\theta)=\pi_\theta(y_t\mid\tau_{<t})/\pi_{\theta_{\mathrm{old}}}(y_t\mid\tau_{<t})$ the policy probability ratio, where $\tau_{<t}$ includes the task and interaction context before the generated token $y_t$. Define the clipped maximization objective
\begin{equation}
    \mathcal J_{\mathrm{clip}}
    = \mathbb E_t\!\left[
    \min\!\left\{
    r_t(\theta)A_t,
    \operatorname{clip}\!\left(r_t(\theta),1-\epsilon_{\text{low}},1+\epsilon_{\text{high}}\right)A_t
    \right\}\right].
    \label{eq:variant-policy-surrogate}
\end{equation}
The expectation follows the base policy loss's mask and reduction over generated tokens.

\paragraph{Routing reward.}
We augment the standard advantage with a routing inconsistency penalty:
\begin{equation}
    \widetilde A_t
    = \widehat A_t-\lambda_{\mathrm{reward}}(1-\kappa_t).
    \label{eq:variant-routing-reward}
\end{equation}
Lower consistency decreases the token's advantage, while perfect agreement adds no penalty. Despite its name, \texttt{routing\_reward} does not recompute terminal task rewards or return targets.

\paragraph{Routing clipping.}
Another way to add a routing inconsistency penalty to the RL objective is through clipping. We define a scaling factor based on routing inconsistency:
\begin{equation}
    \alpha_t = 1 - \lambda_{\text{clip}}(1 - \kappa_t),
\end{equation}
and apply this factor to the clipping bounds:
\begin{equation}
    \epsilon_{\text{low}}' = \alpha_t \cdot \epsilon_{\text{low}},
    \quad
    \epsilon_{\text{high}}' = \alpha_t \cdot \epsilon_{\text{high}}.
\end{equation}
For $\lambda_{\mathrm{clip}}\in[0,1]$, inconsistent routing narrows the clipping interval around one, while $\kappa_t=1$ preserves the baseline interval. This does not impose a hard bound on the actual policy update.

However, neither variant above effectively changed the routing statistics in our experiments. Because $\kappa_t$ is computed from discrete recorded selections, neither provides a differentiated consistency signal to the current router. Ordinary policy-loss gradients can still reach router parameters through differentiable expert-combination weights, as discussed in Appendix~\ref{appendix:grad_flow}, but they do not explicitly reinforce the consensus experts. Thus, adjusting advantages or clipping bounds alone was insufficient for the intended routing control.

\subsection{Load-Balancing Loss}
\label{appendix:load_balance}

We also evaluate the auxiliary load-balancing loss~\citep{shazeer2017} implemented by the Megatron engine~\citep{megatron-lm}. For a microbatch containing $T$ tokens, define the token frequency $f_{l,e}$ and average expert probability $\bar p_{l,e}$ as
\begin{equation}
    \begin{aligned}
        \mathcal S^{\mathrm{aux}}_{l,t}=\operatorname{TopK}_{e}(p_{l,t}[e],k),
        \quad
        f_{l,e}=\frac{1}{kT}\sum_{t=1}^{T}\mathbb I[e\in\mathcal S^{\mathrm{aux}}_{l,t}],
        \quad
        \text{and}~~
        \bar p_{l,e}=\frac{1}{T}\sum_{t=1}^{T}p_{l,t}[e].
    \end{aligned}
    \label{eq:variant-balance-statistics}
\end{equation}
The above statistics are recomputed from current router probabilities, rather than taken from replayed routes. The minimized loss is
\begin{equation}
    \mathcal L^{\mathrm{load-bal}}
    =\frac{\lambda_{\mathrm{load-bal}}}{|\mathcal{L}_{\text{MoE}}| \cdot M}
    \sum_{l\in\mathcal L_{\mathrm{MoE}}}
    \sum_{e=1}^{M}f_{l,e}\bar p_{l,e}.
    \label{eq:variant-balance-loss}
\end{equation}
The frequencies $f_{l,e}$ are detached from the gradient path. Differentiation through $\bar p_{l,e}$ discourages further allocation to heavily used experts and favors less-used experts. It applies to router input tokens beyond the generated thinking and tool-use fields.

To understand the effectiveness of the operation-label-free load-balancing loss, we compare it with our routing control framework in Table~\ref{tab:load_balance_specialization}. We set $\lambda_{\text{load-bal}}=10^{-4}$ and use the same hyperparameters for our method as in the main experiments. The load balance (measured by the CV of expert load as defined in Appendix~\ref{appendix:metric_load_balance}) is improved, but the agentic task performance on \texttt{test-challenge} is not better until we apply the turn-level control. This demonstrates that operation labels are still required for MoE models to improve their hard task solving.

\begin{table*}[!h]
\centering
\caption{Load balancing and task performance on AppWorld. Load-balancing loss effectively reduces the CV of expert load, but this does not translate to better empirical performance.}
\label{tab:load_balance_specialization}
\begin{tabular}{@{}>{\raggedright\arraybackslash}p{0.42\linewidth}ccc@{}}
\toprule
\multirow{2}{*}{\textbf{Method}}
& \multirow{2}{*}{\textbf{Expert-Load CV} $\downarrow$}
& \multicolumn{2}{c}{\textbf{AppWorld TGC} (\%) $\uparrow$} \\
\cmidrule(lr){3-4}
& & test-normal & test-challenge \\
\midrule
Base Model
& 1.240 & 38.69 & 19.66 \\
\midrule
GRPO
& 1.241 & 69.64 & 46.28 \\
~~~~+load-bal~\citep{shazeer2017}
& \textbf{1.155} & \textbf{74.40} & 45.56 \\
~~~~+token-level control \& load-bal
& 1.321 & 69.05 & 43.17 \\
~~~~+token-level \& turn-level control (ours)
& 1.295 & 72.62 & \textbf{51.80} \\
\bottomrule
\end{tabular}
\end{table*}

\section{Technical Details}
In this section, we introduce the technical details of our experiments, particularly MoE model architecture (Appendix~\ref{appendix:detail_moe_arch}), agentic RL infrastructure (Appendix~\ref{appendix:detail_infra}), hyperparameters (Appendices~\ref{appendix:detail_infer} and~\ref{appendix:detail_train}), gradient flow design (Appendix~\ref{appendix:grad_flow}), and operation-label detection (Appendix~\ref{appendix:rubric}).

\subsection{MoE Model Architecture}
\label{appendix:detail_moe_arch}

We take \texttt{Qwen3-30B-A3B} as an example to introduce a few interesting facts about MoE model architecture:

\paragraph{Router module.}
The router module in each MoE layer is just a linear projection (\texttt{torch.nn.Linear}). \texttt{Qwen3-30B-A3B} does not even have a bias term. Therefore, the hidden states of the context are more decisive in choosing which experts to use~\citep{wang2026myth}. Our proposed routing control framework actively changes the hidden states of each MoE layer, which contributes to its instability.

\paragraph{When does expert routing happen?}
Inference in MoE models can be divided into prefilling (processing prompts) and decoding (generating new tokens) stages. Expert routing occurs in both stages, and each token in the trajectories is assigned a set of experts. However, we discuss only expert routing during decoding, as multi-turn agentic trajectories pertain only to the decoding stage.

\paragraph{MoE decoding fusion.}
Production inference engines such as vLLM~\citep{kwon2023efficient} and SGLang~\citep{zheng2024sglang} typically fuse the routing operations (computing router scores + top-$k$ sampling + activating selected experts) into one kernel. Breaking up this kernel will make the inference much less efficient. This also motivates our routing control design, updating the router weights rather than directly assigning expert indices to the routers.

\subsection{Agentic RL Infrastructure}
\label{appendix:detail_infra}

We adapt slime~\citep{slime_github} and extend its default rollout functions to support agents. Agentic trajectories typically contain multiple interactions between agents (generating tokens) and environments (executing code). The number of turns, the token count in each turn, and the execution time are highly variable, which makes standard batched inference impossible. We therefore isolate the decoding of each trajectory into a process, separate from the SGLang server processes. We attach an environment sandbox to each trajectory process, enabling parallel code execution. The high-level idea of the agentic RL rollout framework is summarized in Fig.~\ref{fig:detail_infra}.

\begin{figure}[!h]
    \centering
    \includegraphics[width=\linewidth]{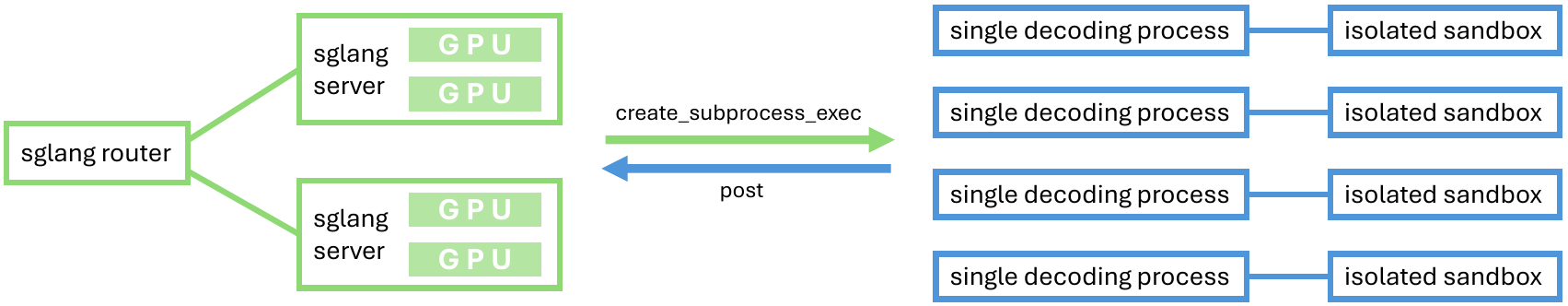}
    \caption{Rollout framework overview. Each trajectory is handled by a separate process with its sandbox for environment execution. SGLang servers generate tokens in response to HTTP requests from trajectory processes.}
    \label{fig:detail_infra}
\end{figure}

\subsection{Inference Settings}
\label{appendix:detail_infer}

We summarize the full inference settings (particularly for evaluation) in Table~\ref{tab:inference_settings}, balancing inference speed and GPU memory capacity (we use B200 nodes\footnote{\url{https://www.nvidia.com/en-us/data-center/dgx-b200/}.} throughout the experiments).

\begin{table}[!h]
\centering
\caption{Evaluation inference settings. Sampling parameters apply at each assistant turn. TP and EP denote tensor and expert parallelism. \textsuperscript{a}AppWorld uses 160 concurrent trajectories, except
for PPO, which uses 80. \textsuperscript{b}AutomationBench uses 144, except for PPO, which uses 40. Concurrency limits are totals across both engines. The explicit context cutoff is checked before each generation turn; an unset cutoff does not remove model or server context limits.}
\label{tab:inference_settings}
\begin{tabular}{@{}>{\raggedright\arraybackslash}p{0.40\linewidth}>{\centering\arraybackslash}p{0.29\linewidth}>{\centering\arraybackslash}p{0.24\linewidth}@{}}
\toprule
\textbf{Setting} & \textbf{AppWorld} & \textbf{AutomationBench} \\
\midrule
Model & \texttt{Qwen3-\allowbreak30B-\allowbreak A3B-\allowbreak Instruct-\allowbreak2507} & \texttt{Qwen3.5-35B-A3B} \\
Inference engine & SGLang & SGLang \\
Inference GPUs & 8 & 8 \\
Engine replicas & 2 & 2 \\
TP / EP per engine & 4 / 4 & 4 / 4 \\
\midrule
Temperature & 0.7 & 0.7 \\
Nucleus sampling ($p$) & 0.8 & 0.8 \\
Top-$k$ sampling & 20 & 20 \\
Evaluation trajectories per task & 1 & 1 \\
Thinking mode & Disabled & Disabled \\
Presence-penalty override & Not set & Not set \\
\midrule
Maximum generated tokens per turn & 1,500 & 2,048 \\
Maximum interaction turns & 50 & 50 \\
Explicit context cutoff (tokens) & Not set & 65,536 \\
Maximum concurrent trajectories
    & 160 / 80\textsuperscript{a}
    & 144 / 40\textsuperscript{b} \\
\bottomrule
\end{tabular}
\end{table}

\subsection{Training Settings}
\label{appendix:detail_train}

For AppWorld~\citep{trivedi2024appworld}, we train \texttt{Qwen3-30B-A3B-Instruct-2507}\footnote{Multi-turn trajectories are very long, exceeding the context window of the older \texttt{Qwen3-30B-A3B}.} for 200 steps. At each step, we sample 10 prompts and generate 16 rollout trajectories per prompt at temperature $1.0$, yielding a nominal batch of 160 trajectories. To reduce waiting for slow rollouts, we terminate the remaining trajectories once every prompt has at least 12 completed trajectories and at least $90\%$ of the full batch (144 trajectories) has completed; completion includes both successful and unsuccessful task outcomes. Aborted trajectories are masked out of the training loss and excluded from group reward normalization. Training trajectories are limited to 30 interaction turns, with at most 1,500 generated tokens per turn and a 3,000-token cap on each environment execution output. We optimize using Adam~\citep{kingma2015adam} with a constant learning rate of $10^{-6}$, gradient-norm clipping at $0.1$, and policy-ratio clipping with $\epsilon=0.2$; both the KL-penalty and entropy-bonus coefficients are set to zero. We evaluate every 10 training steps, allowing up to 50 interaction turns per evaluation trajectory.

\subsection{Injected Gradient Flow in RL Regularization}
\label{appendix:grad_flow}

The RL loss (Eq.~\ref{eq:prelim_policy_gradient}) is computed from the output token probabilities of $\pi_{\bm{\theta}}$. Its gradients can reach router logits indirectly through differentiable expert-combination weights, but do not explicitly promote operation specialization or local routing consistency. We therefore inject auxiliary gradients directly at the router logits $\mathbf{z}_{l,t}$. This supplies explicit routing supervision alongside the standard RL update, without changing the current forward computation.

We introduce $\rho_{l,t}$ to help balance the magnitudes of the RL and routing-control gradients. Let $\mathbf{g}_{l,t}^{\mathrm{in}}$ denote the gradient arriving at $\mathbf{z}_{l,t}$ before auxiliary injection at layer $l$. For a token position $t$ in a turn $\mathcal{T}$, define
\begin{equation}
\rho_{l,t}
= \frac{\|\mathbf{g}_{l,t}^{\mathrm{in}}\|_2}
{\max\!\left\{10^{-8},\,
\max_{u\in\mathcal{T}}\|\mathbf{g}_{l,u}^{\mathrm{in}}\|_2
\right\}}.
\label{eq:control-gradient-scaling}
\end{equation}
This factor lies in $[0,1]$: tokens with weaker incoming gradients receive smaller auxiliary updates, and tokens with zero incoming gradient receive none. The coefficients $\lambda_{\mathrm{turn}}$ and $\lambda_{\mathrm{token}}$ set the overall auxiliary strength, while $\rho_{l,t}$ adjusts it across tokens within each turn. Thus, $\rho_{l,t}$ supports balancing the two training signals but does not by itself enforce equal gradient magnitudes. We compute it once before injection at each layer and share it between both controls.

Using the turn-level objective $J_{\mathrm{turn}}$ from Section~\ref{sec:turn-level-control} and the token-level objective $J_{\mathrm{token}}$ from Section~\ref{sec:token-level-control}, the modified backward rule is
\begin{equation}
\mathbf{g}_{l,t}^{\mathrm{out}}
= \mathbf{g}_{l,t}^{\mathrm{in}}
- g_s\,\rho_{l,t}\left(
\lambda_{\mathrm{turn}}\nabla_{\mathbf{z}_{l,t}}J_{\mathrm{turn}}
+\lambda_{\mathrm{token}}\nabla_{\mathbf{z}_{l,t}}J_{\mathrm{token}}
\right),
\label{eq:control-gradient-combination}
\end{equation}
where $g_s$ is the policy-entropy gate from Section~\ref{sec:entropy-gated-control}. The auxiliary derivatives alter the gradient directions through the full router distribution $\mathbf{p}_{l,t}$. Scaling factors, gates, and discrete expert targets are treated as constants during the backward pass.

We apply this rule at every MoE layer, separately within generated thinking and tool-use fields. A custom autograd operation returns $\mathbf{z}_{l,t}$ from the forward pass and adds the auxiliary gradients in the backward pass. The resulting gradient propagates through $\mathbf{z}_{l,t}=W_l\mathbf{h}_{l,t}$ by the chain rule, updating both the router weights and the representations right before the router. We summarize the actual computation process of our routing control framework in Algorithm~\ref{alg:routing_control}.

\begin{algorithm}[!h]
\caption{Details of Agentic RL with Routing Control}
\label{alg:routing_control}
\begin{algorithmic}[1]
\For{each training step}
    \State Generate trajectories and extract labels and field boundaries.
    \State Compute the RL loss.
    \For{each MoE layer during backpropagation}
        \State Collect routing statistics.
        \State Compute scaling factors from incoming gradients.
        \If{entropy gate is open}
            \State Apply token control to similar adjacent routes within fields.
            \State Apply turn control.
            \State Scale and add the auxiliary updates.
        \EndIf
        \State Continue backpropagation.
    \EndFor
    \State Update model parameters.
\EndFor
\end{algorithmic}
\end{algorithm}

\subsection{Rubrics for Grouping Agentic Operations}
\label{appendix:rubric}

As a key component of the proposed routing-control framework, our rubric for detecting operation labels is simple, deterministic, and task-dependent. We apply hard-coded rules to the auxiliary information in each turn's code field. A turn may contain multiple operations; therefore, its label is the sorted set of unique operation categories detected in that turn. The resulting turn-level labels are assigned to both the thinking and tool-use fields.

\subsubsection{AppWorld Rubrics}
For AppWorld~\citep{trivedi2024appworld}, we extract API calls of the form \texttt{apis.<app>.<api>} from the Python code generated at turn $t$. Let
\begin{equation}
    \mathcal{C}_t^{\mathrm{AW}}
    =
    \{(a_i,p_i)\}_{i=1}^{N_t},
\end{equation}
where $a_i$ and $p_i$ denote the application and API names of the $i$-th call, respectively. We map each call to one of five operation categories:
\begin{equation}
    \mathcal{O}_{\mathrm{AW}}
    =
    \{\textsc{Auth},\textsc{Read},\textsc{Create},
      \textsc{Update},\textsc{Delete}\}.
\end{equation}

The operation classifier follows the ordered rules
\begin{equation}
g_{\mathrm{AW}}(a,p)=
\begin{cases}
\textsc{Auth},
    & p \in \mathcal{P}_{\mathrm{auth}},\\
\textsc{Read},
    & \operatorname{prefix}(p) \in \mathcal{P}_{\mathrm{read}},\\
h\!\left(\mu(a,p)\right),
    & \mu(a,p)\ \text{is available},\\
q\!\left(\operatorname{prefix}(p)\right),
    & \text{otherwise},
\end{cases}
\end{equation}
where $\mathcal{P}_{\mathrm{auth}}$ contains authentication APIs such as
\texttt{login}, \texttt{logout}, and \texttt{signup};
$\mathcal{P}_{\mathrm{read}}$ contains read-oriented prefixes such as
\texttt{show}, \texttt{list}, \texttt{search}, \texttt{get}, and
\texttt{download}; and $\mu(a,p)$ returns the HTTP method associated with the API. The HTTP-method mapping is
\begin{equation}
\begin{aligned}
h(\texttt{GET})&=\textsc{Read}, &\qquad h(\texttt{POST})&=\textsc{Create}, \\
h(\texttt{PATCH/PUT})&=\textsc{Update}, &\qquad h(\texttt{DELETE})&=\textsc{Delete}.
\end{aligned}
\end{equation}
The read-prefix rule is applied before the HTTP-method rule because AppWorld implements several read-only, RPC-style APIs using \texttt{POST}. For example,
\texttt{apis.splitwise.show\_person\_balance(...)} is labeled
\textsc{Read}, despite being implemented as a \texttt{POST} endpoint.

The final multi-label annotation for turn $t$ is
\begin{equation}
    \operatorname{sort}\!\left(
    \operatorname{unique}
    \left\{
    g_{\mathrm{AW}}(a,p)
    \mid (a,p)\in\mathcal{C}_t^{\mathrm{AW}}
    \right\}\right).
\end{equation}
A turn may have multiple operation labels. For example, a turn containing both \texttt{create\_task(...)} and \texttt{delete\_task(...)} gets $\{\textsc{Create},\textsc{Delete}\}$.

\subsubsection{AutomationBench Rubrics}
For AutomationBench~\citep{shepard2026automationbench}, operation detection is performed directly on the structured tool calls generated at turn $t$. Let
\begin{equation}
    \mathcal{C}_t^{\mathrm{Auto}}
    =
    \{(n_i,x_i)\}_{i=1}^{N_t},
\end{equation}
where $n_i$ is the tool name and $x_i$ is its argument dictionary. We use the operation categories
\begin{equation}
    \mathcal{O}_{\mathrm{Auto}}
    =
    \{\textsc{Search},\textsc{Read},\textsc{Create},\textsc{Update},\textsc{Delete},\textsc{None}\}.
\end{equation}
Unlike AppWorld, AutomationBench does not require an external API-to-method lookup: the HTTP method is explicitly included in each \texttt{api\_fetch} call. We define
\begin{equation}
g_{\mathrm{Auto}}(n,x)=
\begin{cases}
\textsc{Search},
    & n=\texttt{api\_search},\\
\textsc{Read},
    & n=\texttt{api\_fetch},\
      x[\texttt{method}]\in\{\texttt{GET},\texttt{HEAD}\},\\
\textsc{Create},
    & n=\texttt{api\_fetch},\
      x[\texttt{method}]=\texttt{POST},\\
\textsc{Update},
    & n=\texttt{api\_fetch},\
      x[\texttt{method}]\in\{\texttt{PATCH},\texttt{PUT}\},\\
\textsc{Delete},
    & n=\texttt{api\_fetch},\
      x[\texttt{method}]=\texttt{DELETE},\\
\textsc{None},
    & \text{otherwise}.
\end{cases}
\end{equation}
Local formatting tools, such as \texttt{base64\_encode}, and malformed or unsupported tool calls are not assigned an operation label.

The turn-level annotation is consequently
\begin{equation}
    \operatorname{sort}\!\left(
    \operatorname{unique}
    \left\{
    g_{\mathrm{Auto}}(n,x)
    \mid (n,x)\in\mathcal{C}_t^{\mathrm{Auto}}
    \right\}\right).
\end{equation}
A turn may have multiple operation labels. For example, a turn containing an \texttt{api\_search} call, a \texttt{GET} request, and two \texttt{POST} requests receives $\{\textsc{Search},\textsc{Read},\textsc{Create}\}$.

\section{Additional Discussion}
In this section, we discuss a few interesting points about MoE routing control. These discussions complement our design choices in Section~\ref{sec:method}.

\subsection{Is It Possible to Manually Fix the Within-Field Routing?}
It is true that production inference engines do not support manually setting expert selections (Appendix~\ref{appendix:detail_moe_arch}), but it is still interesting to examine the effects of doing so. We use the naive inference function from HF Transformers~\citep{wolf-etal-2020-transformers} and manually set the expert selection to be exactly the field-wise consensus set (Eq.~\ref{eq:variant-consensus}).

The results are shown in Table~\ref{tab:fixed_routing_likelihood}. We compute mean log-likelihoods for a set of example AppWorld trajectories using default and fixed routing strategies. We observe that enforcing fixed expert selection substantially degrades model likelihoods. Replacing token-specific routing with each field's majority-selected experts reduces mean token log-probability. These results demonstrate that hard routing constraints can induce a collapse in likelihood.

\begin{table*}[!h]
\centering
\caption{Enforcing fixed expert routing substantially degrades trajectory likelihood. Generation log-likelihoods are computed from a set of example trajectories. Higher mean log-probability is better. $\Delta = 2\cdot(\ell_{\text{fixed}}-\ell_{\text{default}}) / (\ell_{\text{fixed}}+\ell_{\text{default}})$ is the relative decrease.}
\label{tab:fixed_routing_likelihood}
\begin{tabular}{lrrr}
\toprule
\textbf{Trajectory} & \multicolumn{3}{c}{\textbf{Generation Log-Likelihood}} \\
\cmidrule(lr){2-4}
& Default Routing & Fixed Routing & $\Delta$ (\%) \\
\midrule
positive ($R=1$) & -0.2778 & -1.5897 & 140.51 \\
negative ($R=0$) & -0.2795 & -1.5292 & 138.19 \\
\bottomrule
\end{tabular}
\end{table*}

\subsection{Is It Possible to Directly Control the Routers' Entropy?}
\label{appendix:router_entropy}

The stability of MoE RL is closely related to entropy, but the total action entropy from the token distributions is not enough to understand the MoE models' inference behavior. We also examine router entropy.

\paragraph{Entropy decomposition.}
For a fixed context $x$, consider a stochastic router that samples an expert $R\sim q_{\phi}(\cdot\mid x)$ and then the expert samples an action $A\sim\pi_{\theta}(\cdot\mid R,x)$. Here, $A\in\mathbb{V}$ is the next token. We distinguish three quantities: the total action entropy $H(A\mid x)$, the expected action entropy with the route fixed $\mathbb{E}_{r}H(A\mid R=r,x)$, and the router entropy $H(q_{\phi}(\cdot\mid x))$.

\begin{lemma}[Entropy Decomposition]
\label{lem:entropy_decompose}
The total action entropy satisfies
\begin{equation*}
    H(A\mid x)
    \leq \min\!\left\{
        \log|\mathbb{V}|,\;
        \mathbb{E}_{r\sim q_{\phi}(\cdot\mid x)}H(A\mid R=r,x)
        + H\!\left(q_{\phi}(\cdot\mid x)\right)
    \right\}.
\label{eq:router-entropy-bound}
\end{equation*}
\end{lemma}

\noindent\emph{Proof.}
The chain rule gives
\begin{equation*}
    H(A\mid x) = \mathbb{E}_{r}H(A\mid R=r,x) + H(R\mid x) - H(R\mid A,x).
\label{eq:router-entropy-decomposition}
\end{equation*}
Since $H(R\mid A,x)\geq 0$, we can upper-bound the total action entropy by
\begin{equation*}
\begin{aligned}
    H(A\mid x) &\leq \mathbb{E}_{r}H(A\mid R=r,x) + H(R\mid x) \\
    &= \mathbb{E}_{r}H(A\mid R=r,x) + H(q_{\phi}(\cdot\mid x)).
\end{aligned}
\end{equation*}
Given the numerical entropy upper bound $\log|\mathbb{V}|$, the inequality holds. \hfill$\square$

\paragraph{Failure cases of router entropy regularization.}
Fig.~\ref{fig:router-entropy-control} compares training with and without direct regularization of router softmax entropy. The router entropy loss is similar to the standard total action entropy loss, and we set the coefficients to $\pm 10^{-3}$. We observe training collapse when perturbing the routers' entropy. Small router-score changes near the top-$k$ boundary can replace active experts and disrupt subsequent hidden representations, which explains the explosion in total action entropy and training collapse.

\begin{figure}[!h]
    \centering
    \includegraphics[width=\linewidth]{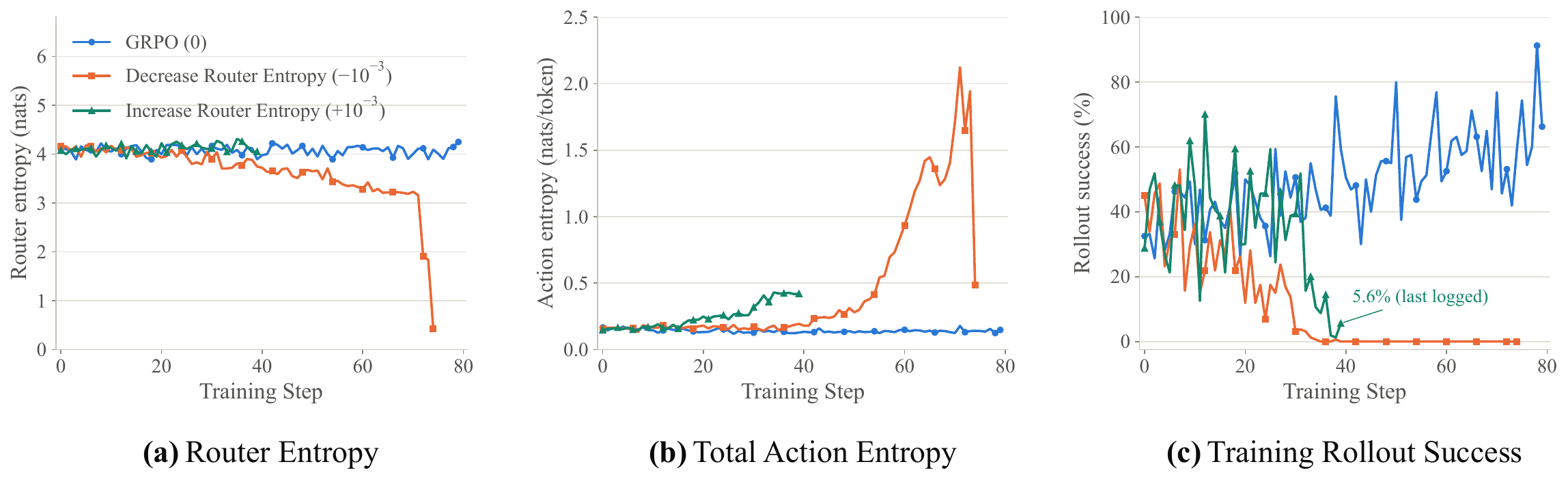}
    \caption{Training curves for GRPO with direct router-entropy regularization. Direct entropy control causes an explosion in total action entropy and training collapse.}
    \label{fig:router-entropy-control}
\end{figure}

\subsection{Why Does Routing Consistency Improve Efficiency?}
\label{appendix:efficiency_reason}

As demonstrated in Section~\ref{sec:exp-efficiency}, encouraging field-wise routing consistency improves the efficiency of MoE models. However, this effect has not been fully understood yet. We propose a controlled analysis here to reveal the mechanism behind this efficiency improvement.

As inference engines such as vLLM~\citep{kwon2023efficient} and SGLang~\citep{zheng2024sglang} encapsulate routing in fused kernels, it is technically hard to manually set the expert selection during inference. Instead, we propose a simple weight matrix mixing trick to reduce the rank of the router's weight matrices. This low-rank mixing can effectively reduce the expert shifts. Specifically, we denote the router's weight matrix as $\mathbf{W} \in \mathbb{R}^{E \times H}$ and the hidden states as $\mathbf{h} \in \mathbb{R}^H$. The router's logit for expert $e$ is computed as:
\begin{equation}
    \text{logit}[e] = \mathbf{W}[e] \cdot \mathbf{h}.
\end{equation}
To manually reduce expert shifts when decoding consecutive tokens, we aim to restrict expert selection to a smaller set. For example, we can force the router logits to be exactly the same for all experts, and then only the first 8 experts will be used. This approach is equivalent to reducing the rank of the routers' weight matrices. Specifically, we categorize experts into $G$ groups according to their indices:
\begin{equation}
    \underbrace{e_1, e_2, ..., e_{K}}_{\text{Group}~1}, \underbrace{e_{K + 1}, e_{K + 2}, ..., e_{2K}}_{\text{Group}~2}, ..., \underbrace{e_{(G-1)K + 1}, e_{(G-1)K + 2}, ..., e_{GK}}_{\text{Group}~G}.
\end{equation}
Then, for each group, its corresponding rows in the low-rank weight matrix are:
\begin{equation}
    \mathbf{W}'[e] = \frac{1}{K}\sum_{e'\in \text{group}(e)} \mathbf{W}[e'],
\end{equation}
where $\text{group}(e)$ is the group to which expert $e$ belongs. This makes the rows within each group exactly the same, fixing any expert selection within a group to be its first expert. Then the new logits are computed as:
\begin{equation}
    \text{logit}'[e] = \mathbf{W}'[e] \cdot \mathbf{h}.
\end{equation}
By varying $G$, we can manually control the number of shifted experts.

Additionally, for small $G$, the number of experts in each group exceeds that in each EP rank, and thus only a portion of GPUs are utilized during inference. We address this bug by re-ranking experts within each bucket defined by group and EP rank:
\begin{equation}
    \mathbf{W}'[e] = 0.5^{\text{tier}(e)} \cdot \frac{1}{K}\sum_{e'\in \text{group}(e)} \mathbf{W}[e'],
\end{equation}
where $\text{tier}(e)$ is the order of expert $e$ in its bucket defined by group and EP rank. Even with $G=1$ (fixed 8 experts), we can now guarantee that all GPUs are utilized.

Based on the above router-collapsing procedure, we compare different routing consistency levels across various batch sizes and expert parallelism (EP) settings~\citep{lepikhingshard}. The results are shown in Fig.~\ref{fig:ep_comparison_shifted}. First, whether EP is enabled does not influence inference efficiency. More importantly, the efficiency gain from more consistent routing benefits only inference with batch sizes larger than 1.

These observations support a batching-based explanation. For a decoding batch containing $B$ tokens with top-$k$ routing, the total number of token–expert assignments remains $Bk$. If these assignments are distributed across fewer active experts, the average number of tokens per active expert increases. This can improve expert-weight reuse and the efficiency of grouped expert computation. At batch size 1, each selected expert receives only one token, so this batching opportunity is absent. The observed batch-size dependence is consistent with this mechanism.

Our intervention is a synthetic diagnostic rather than a quality-preserving inference configuration. We therefore interpret Fig.~\ref{fig:ep_comparison_shifted} as evidence that routing structure can improve batched inference efficiency, which is exactly the real evaluation setting in our experiments.

\begin{figure}[!h]
    \centering
    \includegraphics[width=\linewidth]{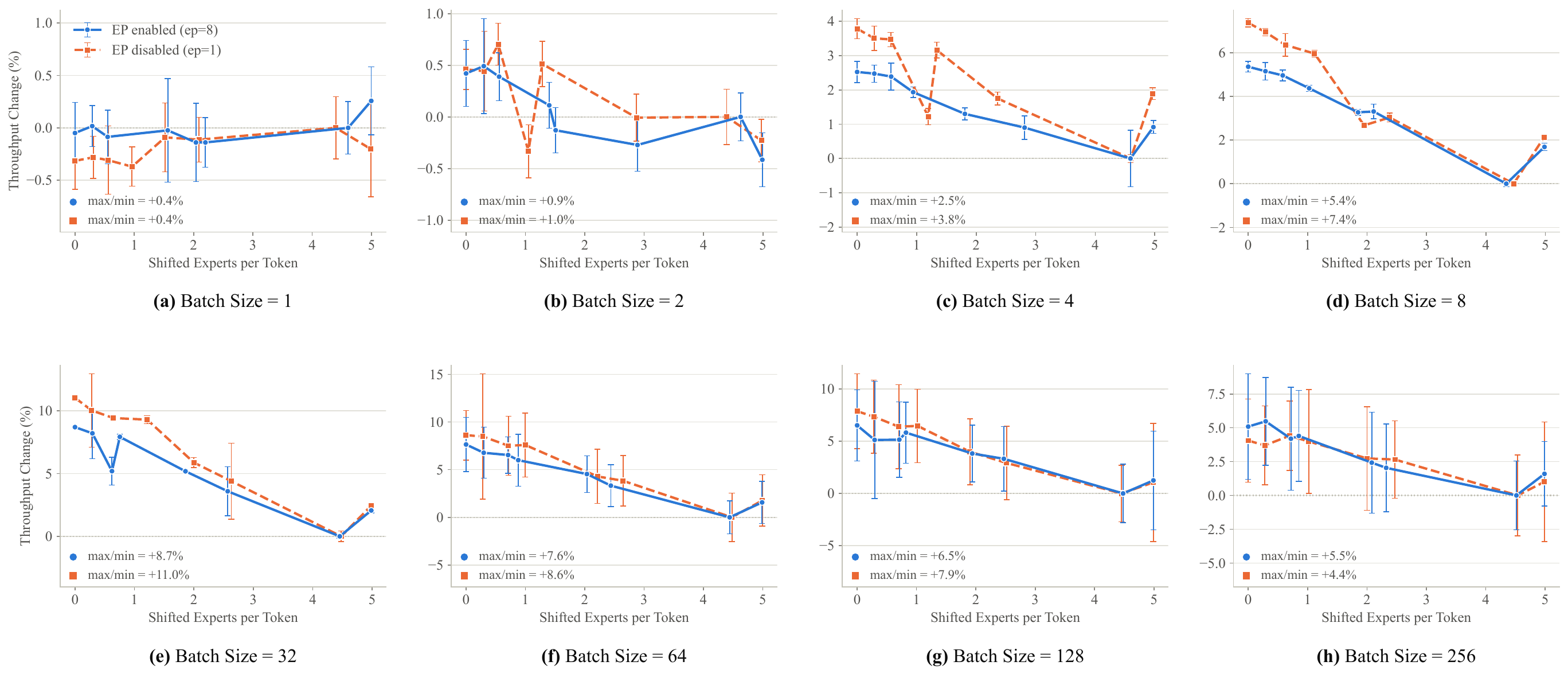}
    \caption{Controlled routing versus inference throughput. We manually collapse the routers' weight matrices to ensure consistent expert selection. The efficiency gain from consistent routing primarily comes from parallelism across a batch of trajectories decoded simultaneously.}
    \label{fig:ep_comparison_shifted}
\end{figure}

\section{Additional Experiments}
In this section, we present the full experimental results to complement Section~\ref{sec:exp} and help explain the effectiveness of each component in our framework as well as other applicable techniques.

\subsection{Ablation Study on Turn Labels}
As in Fig.~\ref{fig:method_pre_analysis}(a), we also compare the full training curves with turn-level routing control based on different labeling schemes in Table~\ref{tab:appworld-label-schemes}. In addition to the operation labels used in the main experiments, we also compare fully random labels and application-name labels. We observe that operation labels provide the best labeling scheme for improving agentic RL performance.

\begin{table}[!h]
\centering
\caption{AppWorld TGC (\%) at training step 199 for \texttt{Qwen3-30B-A3B-Instruct-2507} trained with GRPO and turn-level routing control ($\lambda_{\mathrm{MI}}=2\times10^{-6}$). All three runs share the same saved training settings apart from the labeling schemes. Labeling each turn by its operations yields the best performance.}
\label{tab:appworld-label-schemes}
\begin{tabular}{lcrr}
\toprule
\textbf{Turn Label} & \textbf{\# Classes} & \textbf{test-normal} & \textbf{test-challenge} \\
\midrule
Random           & 5  & 66.67 & 39.33 \\
Application Name & 12 & 66.07 & 40.53 \\
Operation   & 5  & \textbf{70.83} & \textbf{42.69} \\
\bottomrule
\end{tabular}
\end{table}

\subsection{Histogram of the Number of Shifted Experts}
We visualize the histogram of the actual number of expert shifts in Fig.~\ref{fig:expert_shift_histograms}. We observe that, with appropriate routing controls, RL training effectively reduces expert shifts, and more tokens' expert selections are consistent with those of the preceding tokens.

\begin{figure}[!h]
    \centering
    \includegraphics[width=\linewidth]{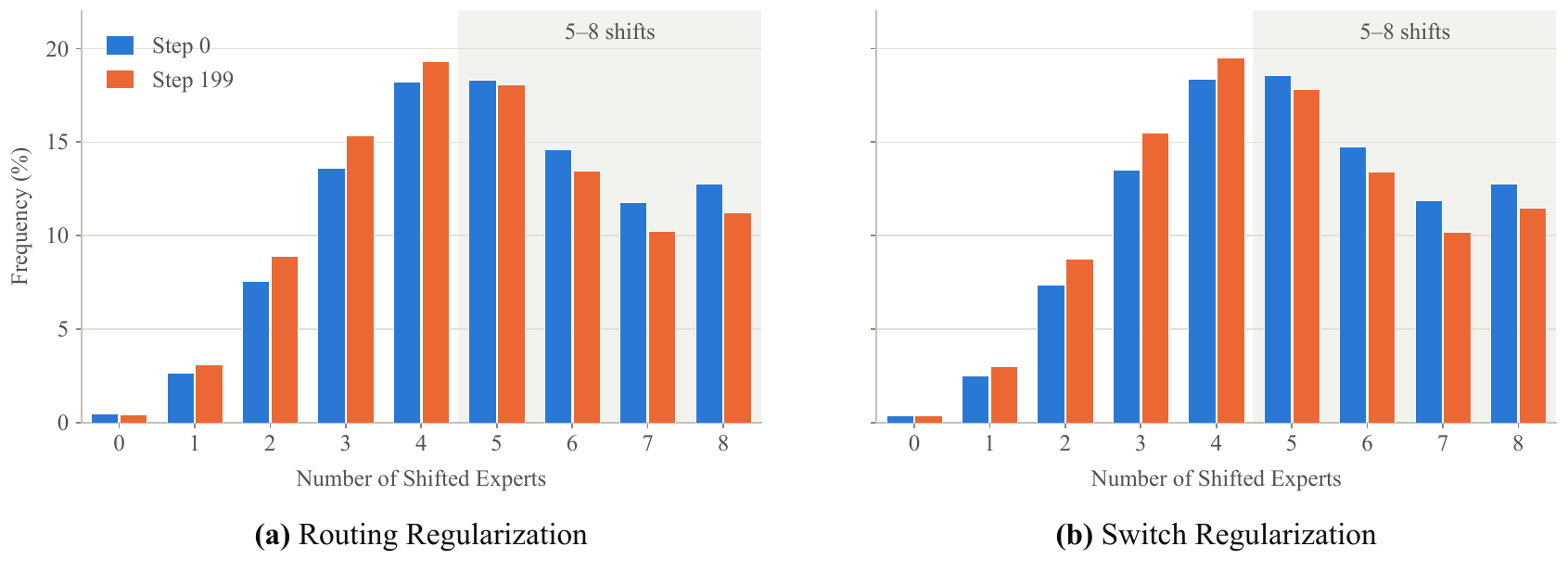}
    \caption{Histogram of expert shifts. After RL training with \texttt{routing\_reg} (Appendix~\ref{appendix:routing-reg}) and \texttt{switch\_reg} (Appendix~\ref{appendix:switch-reg}), high expert-shift counts consistently become less frequent.}
    \label{fig:expert_shift_histograms}
\end{figure}

\subsection{Effect of R3 on Agentic RL}
Rollout Routing Replay (R3)~\citep{ma2026stabilizing} is designed to enforce alignment between inference and training routing and prevent unintended routing drift. We evaluate its effectiveness for agentic RL tasks in our experiments. The results are shown in Fig.~\ref{fig:r3_curves}. Although removing R3 yields slightly better performance, it introduces instability into RL training. For example, without R3, the entropy curve may change sharply with a risk of explosion. Therefore, for better-controlled experiments, R3 is always applied our main experiments.

\begin{figure}[!h]
    \centering
    \includegraphics[width=\linewidth]{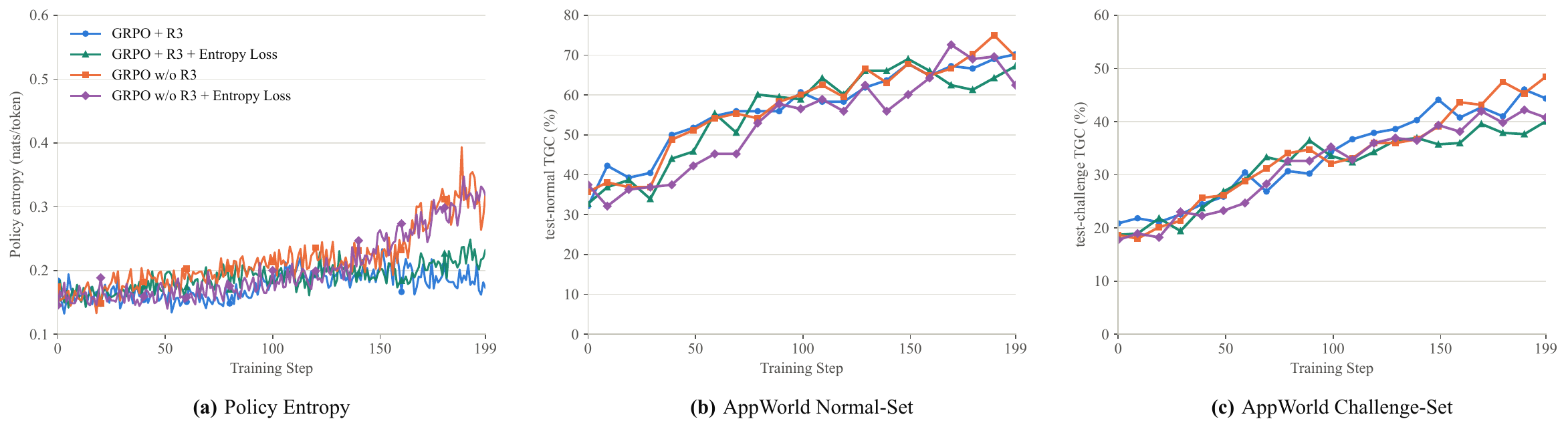}
    \caption{Comparison of training curves with R3. On agentic tasks, R3 can stabilize the entropy dynamics, preventing potential explosion. However, this stability does not translate to better performance due to the loss of RL exploration.}
    \label{fig:r3_curves}
\end{figure}

\subsection{Ablation Study on Entropy Loss}
As we observe that the inference efficiency improvement from our routing control framework often comes with higher total action entropy, we also investigate whether entropy alone can improve efficiency. In Fig.~\ref{fig:entropy_loss_curves}, we directly apply entropy loss to agentic RL and observe that changes in total action entropy do not translate into changes in routing behavior and have no observable effect on inference efficiency.

\begin{figure}[!h]
    \centering
    \includegraphics[width=\linewidth]{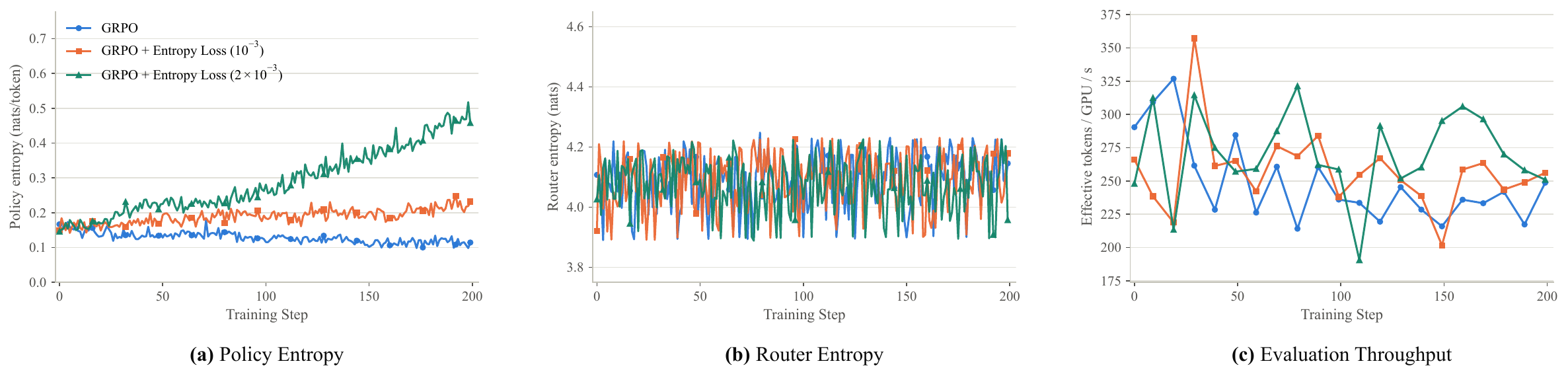}
    \caption{Comparison of training curves with entropy loss. We confirm that entropy loss has no effect on the routers' behavior or inference efficiency.}
    \label{fig:entropy_loss_curves}
\end{figure}

\subsection{Ablation Study on the Token-Level Control Threshold}
To determine the best threshold $h_{\delta}$ as used in Eq.~\ref{eq:token-loss}, we compare different $h_{\delta}$ values in Table~\ref{tab:binary-threshold}. We observe that $h_{\delta} = 0.5$ is the best value. Smaller thresholds reduce our framework to standard MoE-agnostic RL; larger thresholds significantly destabilize RL training.

\begin{table}[!h]
\centering
\caption{Token-level routing control gap thresholds, evaluated using AppWorld TGC (\%) with GRPO. Allowing too few expert shifts reduces our framework to standard RL; allowing too many expert shifts significantly destabilizes RL training.}
\label{tab:binary-threshold}
\begin{tabular}{lrrrrr}
\toprule
\textbf{Threshold} & \textbf{Eval Step} & \textbf{test-normal} & \textbf{test-challenge} & \textbf{Collapse Step} \\
\midrule
$0/8$         & 199 & 64.88 & 35.97 & -- \\ 
$2/8$         & 199 & 61.31 & 39.09 & -- \\ 
$4/8$         & 199 & 66.67 & 42.21 & -- \\ 
$6/8$         & 159 &  0.00 &  0.00 & 149 \\ 
$8/8$         & 139 &  0.00 &  0.00 & 129 \\ 
\bottomrule
\end{tabular}
\end{table}

\subsection{Hyperparameter Magnitude Search}
To empirically determine the best coefficient magnitudes for turn-level and token-level control, we compare a range of coefficients in Tables~\ref{tab:op_mi_hyperparameters} and~\ref{tab:binary_coefficient}, respectively. For both control, we observe optimal magnitude ranges for achieving the best performance and avoiding training collapse. Larger magnitudes require the entropy gate introduced in Section~\ref{sec:entropy-gated-control} to prevent unstable router updates.

\begin{table}[!h]
\centering
\caption{Hyperparameter search for turn-level routing control. We apply only turn-level control to GRPO to determine the appropriate magnitude.}
\label{tab:op_mi_hyperparameters}
\begin{tabular}{lrrrrr}
\toprule
Setting & test-normal & test-challenge
& Entropy & Tokens/GPU/s & Shifted Experts \\
& TGC (\%) & TGC (\%) & & & \\
\midrule
$\lambda_{\text{turn}}=0$ (GRPO)
& 66.67 & 38.85 & 0.1842 & 251.13 & 4.8393 \\
\midrule
$\lambda_{\text{turn}}=10^{-3}$
& 0.00 & 0.00 & 0.2235 & 276.37 & 4.6192 \\
$\lambda_{\text{turn}}=10^{-4}$
& 0.60 & 0.96 & 0.1736 & 223.55 & 4.7340 \\
$\lambda_{\text{turn}}=10^{-5}$
& 63.69 & 36.93 & 0.1801 & 277.63 & 4.8370 \\
$\lambda_{\text{turn}}=10^{-6}$
& 63.69 & 41.73 & 0.1926 & 290.86 & 4.8425 \\
$\lambda_{\text{turn}}=10^{-7}$
& 68.45 & 41.49 & 0.1769 & 278.06 & 4.8247 \\
\bottomrule
\end{tabular}
\end{table}

\begin{table}[!h]
\centering
\caption{Hyperparameter search for token-level routing control. We apply only token-level control to GRPO to determine the appropriate magnitude. TGC is reported at step 199 when available and at the last evaluation otherwise: $\dagger$ denotes step 59 and $\ddagger$ denotes step 79.}
\label{tab:binary_coefficient}
\begin{tabular}{lrrrrr}
\toprule
Setting & test-normal & test-challenge
& Entropy & Tokens/GPU/s & Shifted Experts \\
& TGC (\%) & TGC (\%) & & & \\
\midrule
$\lambda_{\text{token}}=0$ (GRPO)
& 64.88 & 35.97 & 0.1628 & 305.34 & 4.8263 \\
\midrule
$\lambda_{\text{token}}=10^{-4}$\textsuperscript{$\dagger$}
& 0.00 & 0.00 & 0.2649 & 394.55 & 4.5584 \\
$\lambda_{\text{token}}=10^{-5}$\textsuperscript{$\dagger$}
& 0.00 & 0.00 & 0.2639 & 364.24 & 4.6042 \\
$\lambda_{\text{token}}=10^{-6}$\textsuperscript{$\ddagger$}
& 0.00 & 0.00 & 0.2010 & 295.37 & 4.7318 \\
$\lambda_{\text{token}}=10^{-7}$
& 66.67 & 42.21 & 0.1448 & 306.03 & 4.8191 \\
\bottomrule
\end{tabular}
\end{table}

\subsection{Ablation Study on Routing Control Components}
We compare the effect of each individual components of our routing control framwork in Table~\ref{tab:control_ablation}, evaluating the last checkpoints of 200-step training. Turn-level control contributes more to the task performance, and token-level control contributes more to the inference efficiency.

\begin{table*}[!h]
\centering
\setlength\tabcolsep{5pt}
\small
\caption{Individual and combined effects of routing controls on AppWorld. Turn-level control primarily improves model capability on hard tasks, while token-level control primarily improves inference throughput.}
\label{tab:control_ablation}
\begin{tabular}{lrrrrr}
\toprule
\textbf{Setting} & \textbf{test-normal} & \textbf{test-challenge}
& \textbf{Entropy} & \textbf{Tokens/GPU/s} & \textbf{Shifted Experts} \\
& TGC (\%) & TGC (\%) & & & \\
\midrule
GRPO
& 66.67 & 38.85 & 0.1842 & 251.13 & 4.8393 \\
~~~~+ Turn-Level Control
& 62.50 & 42.45 & 0.1943 & \textbf{313.52} & 4.8132 \\
~~~~+ Token-Level Control
& 67.86 & 41.73 & 0.2053 & 277.74 & \textbf{4.6676} \\
~~~~+ Both
& \textbf{75.00} & \textbf{49.20} & \textbf{0.2331} & 312.94 & 4.7416 \\
\bottomrule
\end{tabular}
\end{table*}

\subsection{Ablation Study on Rollout Efficiency}
We also compare different measurement of rollout efficiency in Fig.~\ref{fig:efficiency-ablation}, where we compare the actual rollout time. The conclusion persists that our routing control framework, especially the token-level control, significantly improves the rollout efficiency.

\begin{figure}[!h]
    \centering
    \includegraphics[width=\textwidth]{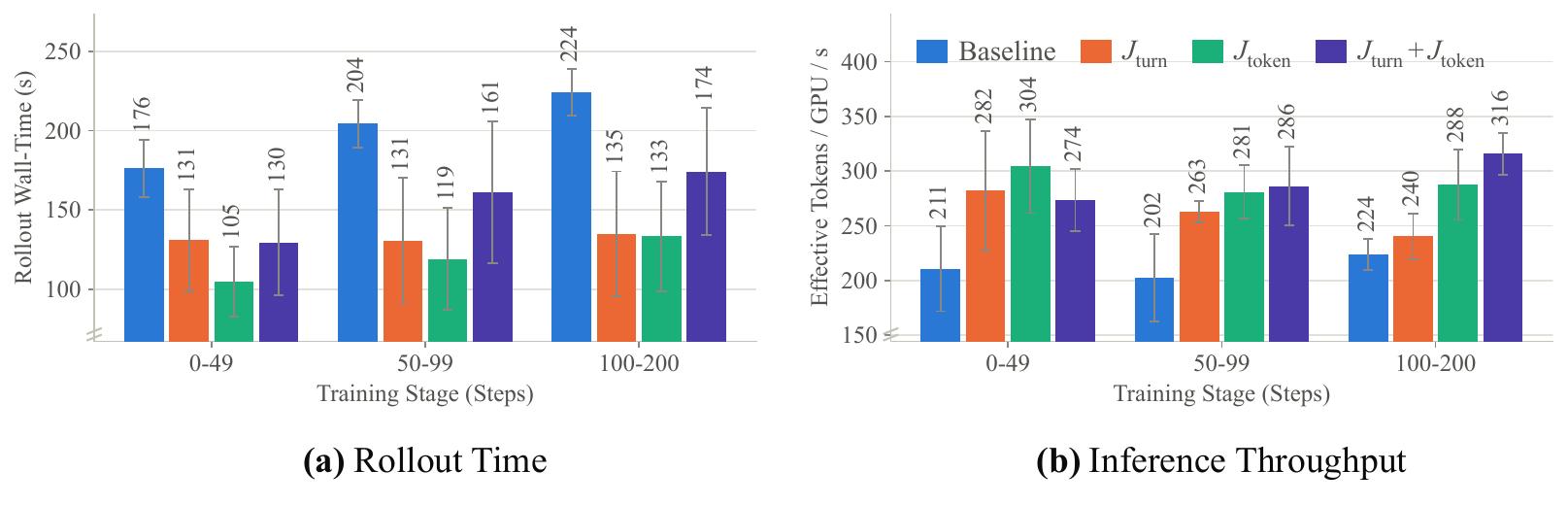}
    \caption{AppWorld rollout efficiency comparison under GRPO with turn-level control, token-level control, or both. (a) Both controls reduce the rollout time for a single trajectory. (b) Both controls improve inference throughput. The token-level control contributes the most to inference efficiency at the early training stage. The reported numbers are averaged over all of the evaluation trajectories collected from the specified training step intervals. The error bars are the standard deviation of all data points.}
    \label{fig:efficiency-ablation}
\end{figure}

\section{Agentic Trajectory Examples}
\definecolor{awtrajblue}{HTML}{3285EC}
\definecolor{awtrajorange}{HTML}{E6A000}
\definecolor{awtrajcode}{HTML}{F4F4F4}
\definecolor{autotrajblue}{HTML}{3285EC}
\definecolor{autotrajorange}{HTML}{E6A000}
\definecolor{autotrajgray}{HTML}{F4F4F4}

\tcbset{awtrajbase/.style={
  enhanced, breakable, boxrule=0.6pt, arc=1.5mm,
  colback=white, coltitle=white, fonttitle=\bfseries\small,
  fontupper=\footnotesize, left=2mm, right=2mm, top=1.5mm, bottom=1.5mm,
  before skip=5pt, after skip=5pt,
  pad at break*=1.5mm,
}}
\newtcolorbox{awtrajagent}[2]{awtrajbase,
  colframe=awtrajblue, colbacktitle=awtrajblue,
  title={Turn #1/20 \quad #2 \hfill Agent},
  title after break={Turn #1/20 \quad #2 \hfill Agent (continued)},
}
\newtcolorbox{awtrajenv}[2]{awtrajbase,
  colframe=awtrajorange, colbacktitle=awtrajorange,
  left skip=3mm,
  title={Turn #1/20 \hfill Environment#2},
  title after break={Turn #1/20 \hfill Environment#2 (continued)},
}
\lstdefinestyle{awtrajcode}{
  language=Python, basicstyle=\ttfamily\scriptsize,
  keywordstyle=\ttfamily, commentstyle=\ttfamily, stringstyle=\ttfamily,
  backgroundcolor=\color{awtrajcode},
  columns=fullflexible, keepspaces=true, showstringspaces=false,
  breaklines=true, breakatwhitespace=false, breakindent=1em,
  tabsize=4, aboveskip=3pt, belowskip=0pt,
  frame=single, rulecolor=\color{black!12}, framerule=0.2pt,
  framesep=3pt, xleftmargin=3pt, xrightmargin=3pt,
}
\lstdefinestyle{awtrajoutput}{
  language={}, basicstyle=\ttfamily\scriptsize,
  columns=fullflexible, keepspaces=true, showstringspaces=false,
  breaklines=true, breakatwhitespace=false, breakindent=1em,
  aboveskip=0pt, belowskip=0pt, frame=none,
}

\tcbset{autotrajbase/.style={
  enhanced, breakable, boxrule=0.6pt, arc=1.5mm,
  colback=white, coltitle=white, fonttitle=\bfseries\small,
  fontupper=\footnotesize, left=2mm, right=2mm, top=1.5mm, bottom=1.5mm,
  before skip=5pt, after skip=5pt, pad at break*=1.5mm,
}}
\newtcolorbox{autotrajagent}[2]{autotrajbase,
  colframe=autotrajblue, colbacktitle=autotrajblue,
  title={Turn #1/17 \quad #2 \hfill Agent},
  title after break={Turn #1/17 \quad #2 \hfill Agent (continued)},
}
\newtcolorbox{autotrajenv}[2]{autotrajbase,
  colframe=autotrajorange, colbacktitle=autotrajorange, left skip=3mm,
  title={Turn #1/17 \hfill Environment#2},
  title after break={Turn #1/17 \hfill Environment#2 (continued)},
}
\lstdefinestyle{autotrajcode}{
  language={}, basicstyle=\ttfamily\scriptsize,
  backgroundcolor=\color{autotrajgray},
  columns=fullflexible, keepspaces=true, showstringspaces=false,
  breaklines=true, breakatwhitespace=false, breakindent=1em,
  tabsize=4, aboveskip=3pt, belowskip=3pt,
  frame=single, rulecolor=\color{black!12}, framerule=0.2pt,
  framesep=3pt, xleftmargin=3pt, xrightmargin=3pt,
}
\lstdefinestyle{autotrajoutput}{
  language={}, basicstyle=\ttfamily\scriptsize,
  columns=fullflexible, keepspaces=true, showstringspaces=false,
  breaklines=true, breakatwhitespace=false, breakindent=1em,
  aboveskip=3pt, belowskip=3pt, frame=none,
}

\subsection{AppWorld  Trajectory}
\label{appendix:appworld_trajectory}

We show a successful AppWorld trajectory with all 20 interaction turns.
Each blue box contains the agent's thinking and tool-use fields; the orange
box below it contains the corresponding environment feedback. Agent text and
Python code are retained, with Markdown delimiters and control tokens removed.
Long environment outputs are excerpted, with omissions marked explicitly;
passwords and access tokens are replaced by placeholders.

\paragraph{Task summary.}
Like all Venmo transactions from today involving any of the user's roommates
on their Venmo social feed. The environment date is \texttt{2023-05-18}.

\begin{center}
\small
\begin{tabular}{ll@{\qquad}ll}
\textbf{Task ID} & \texttt{2a163ab\_1} & \textbf{Sample index} & 21 \\
\textbf{Difficulty} & 2 & \textbf{Turns} & 20 \\
\textbf{Reward} & 1.0 & \textbf{Tests passed} & 6/6 \\
\end{tabular}
\end{center}

\paragraph{Operation labels.}
We label each turn using the API calls in its code field, following
Appendix~C.6.1, and assign the same label set to both generated fields.
Documentation queries are \texttt{Read}, including queries about login or
transaction-liking APIs. Login calls are \texttt{Auth}.
The \texttt{like\_transaction} and \texttt{complete\_task} endpoints use
\texttt{POST} and are therefore \texttt{Create} under the HTTP-method rule.
Turn~16 performs only local computation and has the empty label set
$\emptyset$. Labels are determined by the API calls, including calls that
return an error.

\paragraph{Outcome.}
The agent identifies three roommates, retrieves 50 social-feed transactions,
and selects four transactions matching the date and participants.
Although its string-based success counter prints zero in Turn~18, a repeated
like request in Turn~19 reports that the transaction has already been liked.
The final evaluation records successful completion with all six tests passed.
We retain this discrepancy between the agent's interpretation and the
environment state in the transcript.

\begin{awtrajagent}{1}{$\{\mathtt{Read}\}$}
\textbf{Thinking.} First, let's check what apps are available to us.

\textbf{Tool use.}
\begin{lstlisting}[style=awtrajcode]
print(apis.api_docs.show_app_descriptions())
\end{lstlisting}
\end{awtrajagent}

\begin{awtrajenv}{1}{ (excerpt)}
\begin{lstlisting}[style=awtrajoutput]
[
  {"name": "api_docs", "description": "An app to search and explore API documentation."},
  {"name": "supervisor", "description": "An app to access supervisor's personal information, account credentials, addresses, payment cards, and manage the assigned task."},
  {"name": "phone", "description": "An app to find and manage contact information for friends, family members, etc., send and receive messages, and manage alarms."},
  {"name": "venmo", "description": "A social payment app to send, receive and request money to and from others."},
  ... [other app descriptions omitted]
]
\end{lstlisting}
\end{awtrajenv}

\begin{awtrajagent}{2}{$\{\mathtt{Read}\}$}
\textbf{Thinking.} Now, let's check the available APIs for the \texttt{venmo} app.

\textbf{Tool use.}
\begin{lstlisting}[style=awtrajcode]
print(apis.api_docs.show_api_descriptions(app_name='venmo'))
\end{lstlisting}
\end{awtrajagent}

\begin{awtrajenv}{2}{ (excerpt)}
\begin{lstlisting}[style=awtrajoutput]
[
  {"name": "login", "description": "Login to your account."},
  {"name": "like_transaction", "description": "Like a transaction."},
  {"name": "show_social_feed", "description": "Show your social feed (transactions of your friends)."},
  ... [other API descriptions omitted]
]
\end{lstlisting}
\end{awtrajenv}

\begin{awtrajagent}{3}{$\{\mathtt{Read}\}$}
\textbf{Thinking.} We need to like all the venmo transactions from today involving any of my roommates on my venmo social feed. First, let's get the list of my roommates using the \texttt{phone} app.

\textbf{Tool use.}
\begin{lstlisting}[style=awtrajcode]
print(apis.api_docs.show_api_descriptions(app_name='phone'))
\end{lstlisting}
\end{awtrajagent}

\begin{awtrajenv}{3}{ (excerpt)}
\begin{lstlisting}[style=awtrajoutput]
[
  {"name": "login", "description": "Login to your account."},
  {"name": "show_contact_relationships", "description": "Get a list of all relationships available in your contact book."},
  {"name": "search_contacts", "description": "Search your contact book for relatives' information."},
  ... [other API descriptions omitted]
]
\end{lstlisting}
\end{awtrajenv}

\begin{awtrajagent}{4}{$\{\mathtt{Read}\}$}
\textbf{Thinking.} Now, let's check the \texttt{show\_contact\_relationships} API to get the list of my roommates.

\textbf{Tool use.}
\begin{lstlisting}[style=awtrajcode]
print(apis.api_docs.show_api_doc(app_name='phone', api_name='show_contact_relationships'))
\end{lstlisting}
\end{awtrajagent}

\begin{awtrajenv}{4}{ (excerpt)}
\begin{lstlisting}[style=awtrajoutput]
{
  "app_name": "phone",
  "api_name": "show_contact_relationships",
  "path": "/phone/contact_relationships",
  "method": "GET",
  "description": "Get a list of all relationships available in your contact book.",
  ... [parameters and response schemas omitted]
}
\end{lstlisting}
\end{awtrajenv}

\begin{awtrajagent}{5}{$\{\mathtt{Read}\}$}
\textbf{Thinking.} We need to log in to the phone app first to get the access token. Let's check the login API.

\textbf{Tool use.}
\begin{lstlisting}[style=awtrajcode]
print(apis.api_docs.show_api_doc(app_name='phone', api_name='login'))
\end{lstlisting}
\end{awtrajagent}

\begin{awtrajenv}{5}{ (excerpt)}
\begin{lstlisting}[style=awtrajoutput]
{
  "app_name": "phone",
  "api_name": "login",
  "path": "/phone/auth/token",
  "method": "POST",
  "description": "Login to your account.",
  ... [parameters and response schemas omitted]
}
\end{lstlisting}
\end{awtrajenv}

\begin{awtrajagent}{6}{$\{\mathtt{Read}\}$}
\textbf{Thinking.} We need the password for the phone app. Let's use the \texttt{supervisor} app to get it.

\textbf{Tool use.}
\begin{lstlisting}[style=awtrajcode]
passwords = apis.supervisor.show_account_passwords()
print(passwords)
\end{lstlisting}
\end{awtrajagent}

\begin{awtrajenv}{6}{ (excerpt)}
\begin{lstlisting}[style=awtrajoutput]
[
  ... [other account entries omitted]
  {"account_name": "phone", "password": "[omitted]"},
  {"account_name": "venmo", "password": "[omitted]"}
]
\end{lstlisting}
\end{awtrajenv}

\begin{awtrajagent}{7}{$\{\mathtt{Auth}\}$}
\textbf{Thinking.} Now we have the password for the phone app. Let's log in.

\textbf{Tool use.}
\begin{lstlisting}[style=awtrajcode]
phone_password = [p["password"] for p in passwords if p["account_name"] == "phone"][0]
login_result = apis.phone.login(username="3383946795", password=phone_password)
access_token = login_result["access_token"]
print(access_token)
\end{lstlisting}
\end{awtrajagent}

\begin{awtrajenv}{7}{ (excerpt)}
\begin{lstlisting}[style=awtrajoutput]
[phone access token omitted]
\end{lstlisting}
\end{awtrajenv}

\begin{awtrajagent}{8}{$\{\mathtt{Read}\}$}
\textbf{Thinking.} Now that we have the access token, let's get the list of contact relationships to find my roommates.

\textbf{Tool use.}
\begin{lstlisting}[style=awtrajcode]
contact_relationships = apis.phone.show_contact_relationships(access_token=access_token)
print(contact_relationships)
\end{lstlisting}
\end{awtrajagent}

\begin{awtrajenv}{8}{}
\begin{lstlisting}[style=awtrajoutput]
["coworker", "father", "friend", "manager", "mother", "parent", "roommate", "sibling", "sister"]
\end{lstlisting}
\end{awtrajenv}

\begin{awtrajagent}{9}{$\{\mathtt{Read}\}$}
\textbf{Thinking.} The relationship "roommate" is present. Now, let's get the list of roommates using the \texttt{search\_contacts} API.

\textbf{Tool use.}
\begin{lstlisting}[style=awtrajcode]
print(apis.api_docs.show_api_doc(app_name='phone', api_name='search_contacts'))
\end{lstlisting}
\end{awtrajagent}

\begin{awtrajenv}{9}{ (excerpt)}
\begin{lstlisting}[style=awtrajoutput]
{
  "app_name": "phone",
  "api_name": "search_contacts",
  "path": "/phone/contacts",
  "method": "GET",
  "description": "Search your contact book for relatives' information.",
  ... [parameters and response schemas omitted]
}
\end{lstlisting}
\end{awtrajenv}

\begin{awtrajagent}{10}{$\{\mathtt{Read}\}$}
\textbf{Thinking.} Let's search for contacts with the relationship "roommate".

\textbf{Tool use.}
\begin{lstlisting}[style=awtrajcode]
roommates = apis.phone.search_contacts(access_token=access_token, relationship="roommate")
print(roommates)
\end{lstlisting}
\end{awtrajagent}

\begin{awtrajenv}{10}{ (excerpt)}
\begin{lstlisting}[style=awtrajoutput]
[
  {
    "contact_id": 1660,
    "first_name": "Anthony",
    "last_name": "Harrison",
    "email": "an-harrison@gmail.com",
    "relationships": [
      "roommate"
    ],
    ... [other contact fields omitted]
  },
  {
    "contact_id": 1659,
    "first_name": "Anita",
    "last_name": "Burch",
    "email": "anita.burch@gmail.com",
    "relationships": [
      "roommate"
    ],
    ... [other contact fields omitted]
  },
  {
    "contact_id": 1661,
    "first_name": "Nicholas",
    "last_name": "Weber",
    "email": "nicholas.weber@gmail.com",
    "relationships": [
      "roommate"
    ],
    ... [other contact fields omitted]
  }
]
\end{lstlisting}
\end{awtrajenv}

\begin{awtrajagent}{11}{$\{\mathtt{Read}\}$}
\textbf{Thinking.} Now we have the list of roommates. Next, we need to log in to the Venmo app to access the social feed.

Let's check the login API for Venmo.

\textbf{Tool use.}
\begin{lstlisting}[style=awtrajcode]
print(apis.api_docs.show_api_doc(app_name='venmo', api_name='login'))
\end{lstlisting}
\end{awtrajagent}

\begin{awtrajenv}{11}{ (excerpt)}
\begin{lstlisting}[style=awtrajoutput]
{
  "app_name": "venmo",
  "api_name": "login",
  "path": "/venmo/auth/token",
  "method": "POST",
  "description": "Login to your account.",
  ... [parameters and response schemas omitted]
}
\end{lstlisting}
\end{awtrajenv}

\begin{awtrajagent}{12}{$\{\mathtt{Auth}\}$}
\textbf{Thinking.} We have the email \texttt{mel.bailey@gmail.com} and we already retrieved the Venmo password from the supervisor app earlier. Let's log in to Venmo.

\textbf{Tool use.}
\begin{lstlisting}[style=awtrajcode]
venmo_password = [p["password"] for p in passwords if p["account_name"] == "venmo"][0]
login_result = apis.venmo.login(username="mel.bailey@gmail.com", password=venmo_password)
venmo_access_token = login_result["access_token"]
print(venmo_access_token)
\end{lstlisting}
\end{awtrajagent}

\begin{awtrajenv}{12}{ (excerpt)}
\begin{lstlisting}[style=awtrajoutput]
[Venmo access token omitted]
\end{lstlisting}
\end{awtrajenv}

\begin{awtrajagent}{13}{$\{\mathtt{Read}\}$}
\textbf{Thinking.} Now that we have the Venmo access token, let's retrieve the social feed to find transactions from today involving my roommates.

First, let's check the \texttt{show\_social\_feed} API.

\textbf{Tool use.}
\begin{lstlisting}[style=awtrajcode]
print(apis.api_docs.show_api_doc(app_name='venmo', api_name='show_social_feed'))
\end{lstlisting}
\end{awtrajagent}

\begin{awtrajenv}{13}{ (excerpt)}
\begin{lstlisting}[style=awtrajoutput]
{
  "app_name": "venmo",
  "api_name": "show_social_feed",
  "path": "/venmo/social_feed",
  "method": "GET",
  "description": "Show your social feed (transactions of your friends).",
  ... [parameters and response schemas omitted]
}
\end{lstlisting}
\end{awtrajenv}

\begin{awtrajagent}{14}{$\{\mathtt{Read}\}$}
\textbf{Thinking.} Now, let's retrieve the social feed. We'll loop through pages to ensure we get all transactions.

\textbf{Tool use.}
\begin{lstlisting}[style=awtrajcode]
page_index = 0
social_feed = []
while page_index < 10:
    feed_page = apis.venmo.show_social_feed(access_token=venmo_access_token, page_index=page_index)
    if feed_page:
        social_feed.extend(feed_page)
        page_index += 1
    else:
        break
print(len(social_feed))
\end{lstlisting}
\end{awtrajagent}

\begin{awtrajenv}{14}{}
\begin{lstlisting}[style=awtrajoutput]
50
\end{lstlisting}
\end{awtrajenv}

\begin{awtrajagent}{15}{$\{\mathtt{Read}\}$}
\textbf{Thinking.} We have retrieved 50 transactions from the social feed. Now, let's filter these transactions to find the ones from today involving my roommates.

First, let's get today's date.

\textbf{Tool use.}
\begin{lstlisting}[style=awtrajcode]
current_date_time = apis.phone.get_current_date_and_time()
print(current_date_time)
\end{lstlisting}
\end{awtrajagent}

\begin{awtrajenv}{15}{}
\begin{lstlisting}[style=awtrajoutput]
{"date": "Thursday, May 18, 2023", "time": "09:26 PM"}
\end{lstlisting}
\end{awtrajenv}

\begin{awtrajagent}{16}{$\emptyset$}
\textbf{Thinking.} Now, let's filter the transactions from today involving my roommates. We'll check the \texttt{created\_at} field of each transaction and compare it with today's date.

\textbf{Tool use.}
\begin{lstlisting}[style=awtrajcode]
from datetime import datetime

# Extract today's date in the format 'YYYY-MM-DD'
today_date = "2023-05-18"

# List of roommates' emails
roommate_emails = [roommate["email"] for roommate in roommates]

# Filter transactions from today involving roommates
transactions_to_like = []
for transaction in social_feed:
    created_date = transaction["created_at"].split("T")[0]  # Extract date part
    if created_date == today_date:
        sender_email = transaction["sender"]["email"]
        receiver_email = transaction["receiver"]["email"]
        if sender_email in roommate_emails or receiver_email in roommate_emails:
            transactions_to_like.append(transaction)

print(len(transactions_to_like))
\end{lstlisting}
\end{awtrajagent}

\begin{awtrajenv}{16}{}
\begin{lstlisting}[style=awtrajoutput]
4
\end{lstlisting}
\end{awtrajenv}

\begin{awtrajagent}{17}{$\{\mathtt{Read}\}$}
\textbf{Thinking.} We found 4 transactions from today involving my roommates. Now, let's like each of these transactions using the \texttt{like\_transaction} API.

First, let's check the API specification.

\textbf{Tool use.}
\begin{lstlisting}[style=awtrajcode]
print(apis.api_docs.show_api_doc(app_name='venmo', api_name='like_transaction'))
\end{lstlisting}
\end{awtrajagent}

\begin{awtrajenv}{17}{ (excerpt)}
\begin{lstlisting}[style=awtrajoutput]
{
  "app_name": "venmo",
  "api_name": "like_transaction",
  "path": "/venmo/transactions/{transaction_id}/likes",
  "method": "POST",
  "description": "Like a transaction.",
  ... [parameters and response schemas omitted]
}
\end{lstlisting}
\end{awtrajenv}

\begin{awtrajagent}{18}{$\{\mathtt{Create}\}$}
\textbf{Thinking.} Now, let's like each of the transactions.

\textbf{Tool use.}
\begin{lstlisting}[style=awtrajcode]
liked_count = 0
for transaction in transactions_to_like:
    transaction_id = transaction["transaction_id"]
    result = apis.venmo.like_transaction(transaction_id=transaction_id, access_token=venmo_access_token)
    if result.get("message") == "Transaction liked successfully":
        liked_count += 1

print(liked_count)
\end{lstlisting}
\end{awtrajagent}

\begin{awtrajenv}{18}{}
\begin{lstlisting}[style=awtrajoutput]
0
\end{lstlisting}
\end{awtrajenv}

\begin{awtrajagent}{19}{$\{\mathtt{Create}\}$}
\textbf{Thinking.} It seems that none of the transactions were liked successfully. Let's check the response from the \texttt{like\_transaction} API to understand the issue.

\textbf{Tool use.}
\begin{lstlisting}[style=awtrajcode]
# Let's try liking one transaction and check the response
transaction_id = transactions_to_like[0]["transaction_id"]
result = apis.venmo.like_transaction(transaction_id=transaction_id, access_token=venmo_access_token)
print(result)
\end{lstlisting}
\end{awtrajagent}

\begin{awtrajenv}{19}{ (excerpt)}
\begin{lstlisting}[style=awtrajoutput]
Execution failed. Traceback:
... [stack trace omitted]
Exception: Response status code is 422:
{"message":"You have already liked this transaction."}
\end{lstlisting}
\end{awtrajenv}

\begin{awtrajagent}{20}{$\{\mathtt{Create}\}$}
\textbf{Thinking.} The error indicates that we have already liked these transactions. This explains why the \texttt{like\_transaction} API returned a failure. 

Since the task is to like all the venmo transactions from today involving any of my roommates, and we have already liked them, the task is complete.

Let's mark the task as complete.

\textbf{Tool use.}
\begin{lstlisting}[style=awtrajcode]
apis.supervisor.complete_task()
\end{lstlisting}
\end{awtrajagent}

\begin{awtrajenv}{20}{}
\begin{lstlisting}[style=awtrajoutput]
Execution successful.
\end{lstlisting}
\end{awtrajenv}

\subsection{AutomationBench Trajectory}
We show a successful trajectory on the AutomationBench task \texttt{operations.pipefy\_gmail\_vendor\_approval} from \texttt{automation-public-eval}, rollout~0. The example contains 17 assistant turns, including the final response, and 35 tool calls. Calls issued together remain in the same turn, with feedback listed in the corresponding order. We retain the recorded response-field text, tool-use arguments, and final response; lengthy API-search outputs and irrelevant email fields are explicitly omitted. The heading ``Thinking'' displays the recorded response field; empty fields are marked explicitly. XML tool-call wrappers are removed and their arguments are displayed as JSON, retaining JSON-string parameters. JSON whitespace and Markdown formatting are adjusted for readability.

\paragraph{Task instruction (summary).}
Process unread vendor-review emails from \texttt{procurement@company.example.com}. Match the vendors to their Pipefy cards in \texttt{tbl\_ops}, update each card's status and phase according to the review decision, mark the emails as read, and send confirmation emails to procurement. This is a summary from the trajectory and evaluation assertions; the original task prompt is not included in the log.

\begin{center}
\small
\begin{tabular}{ll@{\qquad}ll}
\textbf{Model} & \texttt{Qwen3.5-35B-A3B} & \textbf{Rollout index} & 0 \\
\textbf{Domain} & Operations & \textbf{Sample index} & 53 \\
\textbf{Assistant turns} & 17 & \textbf{Tool calls} & 35 \\
\textbf{Recorded score} & 1.0 & \textbf{Scored assertions passed} & 10/10 \\
\end{tabular}
\end{center}

\paragraph{Operation labels.}
Following Appendix~\ref{appendix:rubric}, \texttt{api\_search} is labeled \texttt{Search}, whereas \texttt{api\_fetch} is labeled by its HTTP method: \texttt{GET} as \texttt{Read} and \texttt{POST} as \texttt{Create}. Thus, the Pipefy field updates and phase moves, Gmail label changes, and email sends are all \texttt{Create} under this rule. The \texttt{base64\_encode} calls in Turn~15 perform local formatting and receive the empty label set $\emptyset$; the final response in Turn~17 also has the empty label set. Each turn's label set applies to both its response and tool-use fields.

\paragraph{Outcome.}
The agent retrieves four vendor-review emails and resolves their corresponding Pipefy cards after an initially empty title lookup. It issues four status updates and four phase moves, marks the four messages as read, and sends four confirmation emails. The log records \texttt{task\_completed\_correctly} $=1.0$, all 10 scored assertions passed, and zero tool errors. Three additional unscored negative assertions also pass, including that the similarly named Northwind card \texttt{card\_902} was not moved. Agent explanations and environment responses are reproduced as recorded.

\begin{autotrajagent}{1}{$\{\mathtt{Search}\}$}
\textbf{Thinking.} \textit{No text recorded.}

\textbf{Call 1: \texttt{api\_search}.}
\begin{lstlisting}[style=autotrajcode]
{
  "query": "unread emails from procurement@company.example.com vendor reviews",
  "top_k": 10
}
\end{lstlisting}

\textbf{Call 2: \texttt{api\_search}.}
\begin{lstlisting}[style=autotrajcode]
{
  "query": "Pipefy tbl_ops cards update phase status",
  "top_k": 5
}
\end{lstlisting}
\end{autotrajagent}

\begin{autotrajenv}{1}{ (excerpts)}
\textbf{Call 1 (excerpt).}
\begin{lstlisting}[style=autotrajoutput]
{
  "results": [
    {
      "id": "quickbooks.bill.create",
      "method": "POST",
      "description": "Enter a new bill from a vendor. Requires a vendor reference and at least one line item, which can be item-based or account-based.",
      "url": "https://quickbooks.api.intuit.com/v3/company/{companyId}/bill",
      ... [other schema fields omitted]
    },
    ... [9 other search results omitted]
  ],
  "count": 10
}
\end{lstlisting}

\textbf{Call 2 (excerpt).}
\begin{lstlisting}[style=autotrajoutput]
{
  "results": [
    {
      "id": "pipefy.cards.move",
      "method": "POST",
      "description": "Relocate a Pipefy card to another phase (column) within a pipe.",
      "url": "https://api.pipefy.com/graphql/cards/{cardId}/move",
      ... [other schema fields omitted]
    },
    {
      "id": "pipefy.cards.updateField",
      "method": "POST",
      "description": "Modify the value of a specific field on a Pipefy card.",
      "url": "https://api.pipefy.com/graphql/cards/{cardId}/fields:update",
      ... [other schema fields omitted]
    },
    ... [3 other search results omitted]
  ],
  "count": 5
}
\end{lstlisting}
\end{autotrajenv}

\begin{autotrajagent}{2}{$\{\mathtt{Search}\}$}
\textbf{Thinking.} \textit{No text recorded.}

\textbf{Call 1: \texttt{api\_search}.}
\begin{lstlisting}[style=autotrajcode]
{
  "query": "messages unread inbox",
  "top_k": 5
}
\end{lstlisting}
\end{autotrajagent}

\begin{autotrajenv}{2}{ (excerpt)}
\textbf{Call 1 (excerpt).}
\begin{lstlisting}[style=autotrajoutput]
{
  "results": [
    ... [1 other search result omitted],
    {
      "id": "gmail.users.messages.modify",
      "method": "POST",
      "description": "Add or remove labels on a specific message. Useful for archiving, toggling read/unread state, or applying custom labels.",
      "url": "https://gmail.googleapis.com/gmail/v1/users/{userId}/messages/{id}/modify",
      ... [other schema fields omitted]
    },
    ... [1 other search result omitted],
    {
      "id": "gmail.users.messages.list",
      "method": "GET",
      "description": "Retrieve messages from the user's mailbox, optionally narrowed by label, sender, subject, or a search query.",
      "url": "https://gmail.googleapis.com/gmail/v1/users/{userId}/messages",
      ... [other schema fields omitted]
    },
    ... [1 other search result omitted]
  ],
  "count": 5
}
\end{lstlisting}
\end{autotrajenv}

\begin{autotrajagent}{3}{$\{\mathtt{Read}\}$}
\textbf{Thinking.} \textit{No text recorded.}

\textbf{Call 1: \texttt{api\_fetch}.}
\begin{lstlisting}[style=autotrajcode]
{
  "method": "GET",
  "url": "https://gmail.googleapis.com/gmail/v1/users/me/messages",
  "params": "{\"q\": \"from:procurement@company.example.com is:unread vendor review\", \"maxResults\": 50}"
}
\end{lstlisting}
\end{autotrajagent}

\begin{autotrajenv}{3}{}
\textbf{Call 1.}
\begin{lstlisting}[style=autotrajoutput]
{"messages": [{"id": "msg_rev_002", "threadId": "thr_rev_002"}, {"id": "msg_rev_003", "threadId": "thr_rev_003"}, {"id": "msg_rev_004", "threadId": "thr_rev_004"}, {"id": "msg_rev_005", "threadId": "thr_rev_005"}], "resultSizeEstimate": 4}
\end{lstlisting}
\end{autotrajenv}

\begin{autotrajagent}{4}{$\{\mathtt{Read}\}$}
\textbf{Thinking.} \textit{No text recorded.}

\textbf{Call 1: \texttt{api\_fetch}.}
\begin{lstlisting}[style=autotrajcode]
{
  "method": "GET",
  "url": "https://gmail.googleapis.com/gmail/v1/users/me/messages/msg_rev_002"
}
\end{lstlisting}

\textbf{Call 2: \texttt{api\_fetch}.}
\begin{lstlisting}[style=autotrajcode]
{
  "method": "GET",
  "url": "https://gmail.googleapis.com/gmail/v1/users/me/messages/msg_rev_003"
}
\end{lstlisting}

\textbf{Call 3: \texttt{api\_fetch}.}
\begin{lstlisting}[style=autotrajcode]
{
  "method": "GET",
  "url": "https://gmail.googleapis.com/gmail/v1/users/me/messages/msg_rev_004"
}
\end{lstlisting}

\textbf{Call 4: \texttt{api\_fetch}.}
\begin{lstlisting}[style=autotrajcode]
{
  "method": "GET",
  "url": "https://gmail.googleapis.com/gmail/v1/users/me/messages/msg_rev_005"
}
\end{lstlisting}
\end{autotrajagent}

\begin{autotrajenv}{4}{ (excerpts)}
\textbf{Call 1 (excerpt).}
\begin{lstlisting}[style=autotrajoutput]
{
  "id": "msg_rev_002",
  "from": "procurement@company.example.com",
  "to": ["ops@company.example.com"],
  "subject": "Vendor Review - Summit",
  "body_plain": "Summit review complete. Decision: Approved.",
  "label_ids": ["INBOX", "UNREAD"],
  ... [other message fields omitted]
}
\end{lstlisting}

\textbf{Call 2 (excerpt).}
\begin{lstlisting}[style=autotrajoutput]
{
  "id": "msg_rev_003",
  "from": "procurement@company.example.com",
  "to": ["ops@company.example.com"],
  "subject": "Vendor Review - Apex",
  "body_plain": "Apex review complete. Decision: Approved.",
  "label_ids": ["INBOX", "UNREAD"],
  ... [other message fields omitted]
}
\end{lstlisting}

\textbf{Call 3 (excerpt).}
\begin{lstlisting}[style=autotrajoutput]
{
  "id": "msg_rev_004",
  "from": "procurement@company.example.com",
  "to": ["ops@company.example.com"],
  "subject": "Vendor Review - NorthWind LLC",
  "body_plain": "NorthWind LLC review complete. Decision: Rejected. Reason: failed background check.",
  "label_ids": ["INBOX", "UNREAD"],
  ... [other message fields omitted]
}
\end{lstlisting}

\textbf{Call 4 (excerpt).}
\begin{lstlisting}[style=autotrajoutput]
{
  "id": "msg_rev_005",
  "from": "procurement@company.example.com",
  "to": ["ops@company.example.com"],
  "subject": "Vendor Review - BlueYard",
  "body_plain": "BlueYard review complete. Decision: Conditional - pending insurance docs.",
  "label_ids": ["INBOX", "UNREAD"],
  ... [other message fields omitted]
}
\end{lstlisting}
\end{autotrajenv}

\begin{autotrajagent}{5}{$\{\mathtt{Search}\}$}
\textbf{Thinking.} Now I have the 4 vendor review decisions:\newline{}
1. \textbf{Summit} - Approved\newline{}
2. \textbf{Apex} - Approved\newline{}
3. \textbf{NorthWind LLC} - Rejected (Reason: failed background check)\newline{}
4. \textbf{BlueYard} - Conditional - pending insurance docs

I need to:\newline{}
1. Find the Pipefy cards for each vendor in the \texttt{tbl\_ops} table\newline{}
2. Update each card with the appropriate phase and status\newline{}
3. Mark the emails as read\newline{}
4. Reply to procurement confirming each decision

Let me first search for the Pipefy table \texttt{tbl\_ops} and find the vendor cards.

\textbf{Call 1: \texttt{api\_search}.}
\begin{lstlisting}[style=autotrajcode]
{
  "query": "Pipefy tbl_ops table records",
  "top_k": 5
}
\end{lstlisting}
\end{autotrajagent}

\begin{autotrajenv}{5}{ (excerpt)}
\textbf{Call 1 (excerpt).}
\begin{lstlisting}[style=autotrajoutput]
{
  "results": [
    {
      "id": "pipefy.database.findRecords",
      "method": "GET",
      "description": "Search for records within a Pipefy database table by title. Returns all matching entries along with their associated field values.",
      "url": "https://api.pipefy.com/graphql/tables/{tableId}/records:find",
      ... [other schema fields omitted]
    },
    ... [4 other search results omitted]
  ],
  "count": 5
}
\end{lstlisting}
\end{autotrajenv}

\begin{autotrajagent}{6}{$\{\mathtt{Search}\}$}
\textbf{Thinking.} I need to find the table ID for \texttt{tbl\_ops}. Let me search for it.

\textbf{Call 1: \texttt{api\_search}.}
\begin{lstlisting}[style=autotrajcode]
{
  "query": "Pipefy tables list tbl_ops",
  "top_k": 5
}
\end{lstlisting}
\end{autotrajagent}

\begin{autotrajenv}{6}{ (excerpt)}
\textbf{Call 1 (excerpt).}
\begin{lstlisting}[style=autotrajoutput]
{
  "results": [
    {
      "id": "pipefy.database.findRecords",
      "method": "GET",
      "description": "Search for records within a Pipefy database table by title. Returns all matching entries along with their associated field values.",
      "url": "https://api.pipefy.com/graphql/tables/{tableId}/records:find",
      ... [other schema fields omitted]
    },
    ... [4 other search results omitted]
  ],
  "count": 5
}
\end{lstlisting}
\end{autotrajenv}

\begin{autotrajagent}{7}{$\{\mathtt{Search}\}$}
\textbf{Thinking.} I need to find the table ID for \texttt{tbl\_ops}. Let me try to search for records with vendor names to find the table. First, let me try searching for one of the vendors to see if I can find the table ID.

Actually, I realize I need to first find the table ID for \texttt{tbl\_ops}. Let me try a different approach - search for the table directly.

\textbf{Call 1: \texttt{api\_search}.}
\begin{lstlisting}[style=autotrajcode]
{
  "query": "Pipefy table tbl_ops vendor",
  "top_k": 5
}
\end{lstlisting}
\end{autotrajagent}

\begin{autotrajenv}{7}{ (excerpt)}
\textbf{Call 1 (excerpt).}
\begin{lstlisting}[style=autotrajoutput]
{
  "results": [
    {
      "id": "pipefy.database.findRecords",
      "method": "GET",
      "description": "Search for records within a Pipefy database table by title. Returns all matching entries along with their associated field values.",
      "url": "https://api.pipefy.com/graphql/tables/{tableId}/records:find",
      ... [other schema fields omitted]
    },
    ... [4 other search results omitted]
  ],
  "count": 5
}
\end{lstlisting}
\end{autotrajenv}

\begin{autotrajagent}{8}{$\{\mathtt{Read}\}$}
\textbf{Thinking.} I need to find the table ID for \texttt{tbl\_ops}. Let me try searching for records with a vendor name to discover the table structure. I'll try searching for "Summit" in a table - but I need the table ID first.

Let me try a common table ID pattern or search for the table. Perhaps the table ID is "tbl\_ops" itself or I need to find it differently.

Let me try searching for records using "tbl\_ops" as the tableId:

\textbf{Call 1: \texttt{api\_fetch}.}
\begin{lstlisting}[style=autotrajcode]
{
  "method": "GET",
  "url": "https://api.pipefy.com/graphql/tables/tbl_ops/records:find",
  "params": "{\"title\": \"Summit\"}"
}
\end{lstlisting}
\end{autotrajagent}

\begin{autotrajenv}{8}{}
\textbf{Call 1.}
\begin{lstlisting}[style=autotrajoutput]
{"records": []}
\end{lstlisting}
\end{autotrajenv}

\begin{autotrajagent}{9}{$\{\mathtt{Read}\}$}
\textbf{Thinking.} No records found with title "Summit" in tbl\_ops. Let me try searching for all records in the table to see what's there, or try different table IDs. Let me first try to find the table by searching for any vendor name.

\textbf{Call 1: \texttt{api\_fetch}.}
\begin{lstlisting}[style=autotrajcode]
{
  "method": "GET",
  "url": "https://api.pipefy.com/graphql/tables/tbl_ops/records:find",
  "params": "{\"title\": \"\"}"
}
\end{lstlisting}
\end{autotrajagent}

\begin{autotrajenv}{9}{}
\textbf{Call 1.}
\begin{lstlisting}[style=autotrajoutput]
{"records": [{"id": "pf_rec_3", "title": "Vendor Review - Northwind", "record_fields": [{"field": {"id": "table_id", "label": "table_id"}, "value": "tbl_ops"}, {"field": {"id": "card_id", "label": "card_id"}, "value": "card_902"}]}, {"id": "pf_rec_4", "title": "Vendor Review - Summit", "record_fields": [{"field": {"id": "table_id", "label": "table_id"}, "value": "tbl_ops"}, {"field": {"id": "card_id", "label": "card_id"}, "value": "card_903"}]}, {"id": "pf_rec_5", "title": "Vendor Review - Apex", "record_fields": [{"field": {"id": "table_id", "label": "table_id"}, "value": "tbl_ops"}, {"field": {"id": "card_id", "label": "card_id"}, "value": "card_904"}]}, {"id": "pf_rec_6", "title": "Vendor Review - NorthWind LLC", "record_fields": [{"field": {"id": "table_id", "label": "table_id"}, "value": "tbl_ops"}, {"field": {"id": "card_id", "label": "card_id"}, "value": "card_905"}]}, {"id": "pf_rec_7", "title": "Vendor Review - BlueYard", "record_fields": [{"field": {"id": "table_id", "label": "table_id"}, "value": "tbl_ops"}, {"field": {"id": "card_id", "label": "card_id"}, "value": "card_906"}]}]}
\end{lstlisting}
\end{autotrajenv}

\begin{autotrajagent}{10}{$\{\mathtt{Search}\}$}
\textbf{Thinking.} I found the records in \texttt{tbl\_ops}. Now I have the card IDs for each vendor:\newline{}
- \textbf{Summit}: card\_903\newline{}
- \textbf{Apex}: card\_904\newline{}
- \textbf{NorthWind LLC}: card\_905\newline{}
- \textbf{BlueYard}: card\_906

Now I need to:\newline{}
1. Move each card to the appropriate phase\newline{}
2. Update the status field for each card\newline{}
3. Mark emails as read\newline{}
4. Reply to procurement

First, let me check what phases are available and find the phase IDs for \texttt{phase\_approved}, \texttt{phase\_rejected}, and \texttt{phase\_pending\_docs}. I also need to find the field ID for \texttt{field\_status}.

Let me try to get more details about the cards to understand the field IDs and phase IDs. I'll need to update the cards with the appropriate phase and status.

For the phases:\newline{}
- Approved vendors (Summit, Apex) $\rightarrow$ phase\_approved with status "Approved"\newline{}
- Rejected vendors (NorthWind LLC) $\rightarrow$ phase\_rejected with status "Rejected"\newline{}
- Conditional vendors (BlueYard) $\rightarrow$ phase\_pending\_docs with status "Pending Docs"

I need to find the phase IDs and field IDs. Let me try to update the cards and see what happens, or search for phase information.

Actually, let me first try to move the cards to the appropriate phases. I'll need the destination\_phase\_id. Let me try common phase IDs or search for them.

Let me try to update the cards with the status field first, and then move them to phases. I'll need to find the field\_id for \texttt{field\_status}.

Let me try a different approach - search for the field information by looking at the card details.

\textbf{Call 1: \texttt{api\_search}.}
\begin{lstlisting}[style=autotrajcode]
{
  "query": "Pipefy cards get details field phase",
  "top_k": 5
}
\end{lstlisting}
\end{autotrajagent}

\begin{autotrajenv}{10}{ (excerpt)}
\textbf{Call 1 (excerpt).}
\begin{lstlisting}[style=autotrajoutput]
{
  "results": [
    {
      "id": "pipefy.cards.move",
      "method": "POST",
      "description": "Relocate a Pipefy card to another phase (column) within a pipe.",
      "url": "https://api.pipefy.com/graphql/cards/{cardId}/move",
      ... [other schema fields omitted]
    },
    ... [1 other search result omitted],
    {
      "id": "pipefy.cards.updateField",
      "method": "POST",
      "description": "Modify the value of a specific field on a Pipefy card.",
      "url": "https://api.pipefy.com/graphql/cards/{cardId}/fields:update",
      ... [other schema fields omitted]
    },
    ... [2 other search results omitted]
  ],
  "count": 5
}
\end{lstlisting}
\end{autotrajenv}

\begin{autotrajagent}{11}{$\{\mathtt{Create}\}$}
\textbf{Thinking.} I need to find the phase IDs and field IDs. Let me try to get more information about the cards or the pipe structure. Since I don't have a direct way to get phase IDs, let me try using common phase ID patterns or attempt the operations with assumed IDs.

Looking at the task requirements:\newline{}
- Approved vendors $\rightarrow$ phase\_approved with status "Approved"\newline{}
- Rejected vendors $\rightarrow$ phase\_rejected with status "Rejected"\newline{}
- Conditional vendors $\rightarrow$ phase\_pending\_docs with status "Pending Docs"

The phase IDs might be: \texttt{phase\_approved}, \texttt{phase\_rejected}, \texttt{phase\_pending\_docs}\newline{}
The field ID for status might be: \texttt{field\_status}

Let me try to update the cards with the status field first, then move them to phases. I'll use the IDs as specified in the task.

Let me start by updating the status field for each card and then moving them to the appropriate phases.

\textbf{Call 1: \texttt{api\_fetch}.}
\begin{lstlisting}[style=autotrajcode]
{
  "method": "POST",
  "url": "https://api.pipefy.com/graphql/cards/card_903/fields:update",
  "body": "{\"field_id\": \"field_status\", \"value\": \"Approved\"}"
}
\end{lstlisting}

\textbf{Call 2: \texttt{api\_fetch}.}
\begin{lstlisting}[style=autotrajcode]
{
  "method": "POST",
  "url": "https://api.pipefy.com/graphql/cards/card_904/fields:update",
  "body": "{\"field_id\": \"field_status\", \"value\": \"Approved\"}"
}
\end{lstlisting}

\textbf{Call 3: \texttt{api\_fetch}.}
\begin{lstlisting}[style=autotrajcode]
{
  "method": "POST",
  "url": "https://api.pipefy.com/graphql/cards/card_905/fields:update",
  "body": "{\"field_id\": \"field_status\", \"value\": \"Rejected\"}"
}
\end{lstlisting}

\textbf{Call 4: \texttt{api\_fetch}.}
\begin{lstlisting}[style=autotrajcode]
{
  "method": "POST",
  "url": "https://api.pipefy.com/graphql/cards/card_906/fields:update",
  "body": "{\"field_id\": \"field_status\", \"value\": \"Pending Docs\"}"
}
\end{lstlisting}
\end{autotrajagent}

\begin{autotrajenv}{11}{}
\textbf{Call 1.}
\begin{lstlisting}[style=autotrajoutput]
{"card": {"id": "pipefy_297a31a99398452687c6e0d7e6b242d8", "title": "", "fields": [{"field": {"id": "field_status", "label": "field_status"}, "value": "Approved"}]}, "success": true}
\end{lstlisting}

\textbf{Call 2.}
\begin{lstlisting}[style=autotrajoutput]
{"card": {"id": "pipefy_ea8f42b848eb4af68feb6d4c1a807ef9", "title": "", "fields": [{"field": {"id": "field_status", "label": "field_status"}, "value": "Approved"}]}, "success": true}
\end{lstlisting}

\textbf{Call 3.}
\begin{lstlisting}[style=autotrajoutput]
{"card": {"id": "pipefy_ea646f731c9f4affaad5c8a7df752e15", "title": "", "fields": [{"field": {"id": "field_status", "label": "field_status"}, "value": "Rejected"}]}, "success": true}
\end{lstlisting}

\textbf{Call 4.}
\begin{lstlisting}[style=autotrajoutput]
{"card": {"id": "pipefy_29d2e5e68e3c4eec8db7cc086592c0ca", "title": "", "fields": [{"field": {"id": "field_status", "label": "field_status"}, "value": "Pending Docs"}]}, "success": true}
\end{lstlisting}
\end{autotrajenv}

\begin{autotrajagent}{12}{$\{\mathtt{Create}\}$}
\textbf{Thinking.} The status fields were updated successfully. Now I need to move each card to the appropriate phase. Let me move them:\newline{}
- card\_903 (Summit) $\rightarrow$ phase\_approved\newline{}
- card\_904 (Apex) $\rightarrow$ phase\_approved\newline{}
- card\_905 (NorthWind LLC) $\rightarrow$ phase\_rejected\newline{}
- card\_906 (BlueYard) $\rightarrow$ phase\_pending\_docs

\textbf{Call 1: \texttt{api\_fetch}.}
\begin{lstlisting}[style=autotrajcode]
{
  "method": "POST",
  "url": "https://api.pipefy.com/graphql/cards/card_903/move",
  "body": "{\"destination_phase_id\": \"phase_approved\"}"
}
\end{lstlisting}

\textbf{Call 2: \texttt{api\_fetch}.}
\begin{lstlisting}[style=autotrajcode]
{
  "method": "POST",
  "url": "https://api.pipefy.com/graphql/cards/card_904/move",
  "body": "{\"destination_phase_id\": \"phase_approved\"}"
}
\end{lstlisting}

\textbf{Call 3: \texttt{api\_fetch}.}
\begin{lstlisting}[style=autotrajcode]
{
  "method": "POST",
  "url": "https://api.pipefy.com/graphql/cards/card_905/move",
  "body": "{\"destination_phase_id\": \"phase_rejected\"}"
}
\end{lstlisting}

\textbf{Call 4: \texttt{api\_fetch}.}
\begin{lstlisting}[style=autotrajcode]
{
  "method": "POST",
  "url": "https://api.pipefy.com/graphql/cards/card_906/move",
  "body": "{\"destination_phase_id\": \"phase_pending_docs\"}"
}
\end{lstlisting}
\end{autotrajagent}

\begin{autotrajenv}{12}{}
\textbf{Call 1.}
\begin{lstlisting}[style=autotrajoutput]
{"card": {"id": "pipefy_fc6d5b11ee0b483da153560bb89f4cd0", "title": "", "current_phase": {"id": "phase_approved", "name": ""}}}
\end{lstlisting}

\textbf{Call 2.}
\begin{lstlisting}[style=autotrajoutput]
{"card": {"id": "pipefy_979ba8c53a9b4be5a73be873004dbe19", "title": "", "current_phase": {"id": "phase_approved", "name": ""}}}
\end{lstlisting}

\textbf{Call 3.}
\begin{lstlisting}[style=autotrajoutput]
{"card": {"id": "pipefy_ade3b513343d4b7e8e1e07c8eeb644f0", "title": "", "current_phase": {"id": "phase_rejected", "name": ""}}}
\end{lstlisting}

\textbf{Call 4.}
\begin{lstlisting}[style=autotrajoutput]
{"card": {"id": "pipefy_4a4438e5e4e14bef9955290a44a60d1e", "title": "", "current_phase": {"id": "phase_pending_docs", "name": ""}}}
\end{lstlisting}
\end{autotrajenv}

\begin{autotrajagent}{13}{$\{\mathtt{Create}\}$}
\textbf{Thinking.} All cards have been moved to the appropriate phases and their status fields have been updated. Now I need to:\newline{}
1. Mark the emails as read\newline{}
2. Reply to procurement confirming each decision

Let me first mark the emails as read by removing the UNREAD label:

\textbf{Call 1: \texttt{api\_fetch}.}
\begin{lstlisting}[style=autotrajcode]
{
  "method": "POST",
  "url": "https://gmail.googleapis.com/gmail/v1/users/me/messages/msg_rev_002/modify",
  "body": "{\"removeLabelIds\": [\"UNREAD\"]}"
}
\end{lstlisting}

\textbf{Call 2: \texttt{api\_fetch}.}
\begin{lstlisting}[style=autotrajcode]
{
  "method": "POST",
  "url": "https://gmail.googleapis.com/gmail/v1/users/me/messages/msg_rev_003/modify",
  "body": "{\"removeLabelIds\": [\"UNREAD\"]}"
}
\end{lstlisting}

\textbf{Call 3: \texttt{api\_fetch}.}
\begin{lstlisting}[style=autotrajcode]
{
  "method": "POST",
  "url": "https://gmail.googleapis.com/gmail/v1/users/me/messages/msg_rev_004/modify",
  "body": "{\"removeLabelIds\": [\"UNREAD\"]}"
}
\end{lstlisting}

\textbf{Call 4: \texttt{api\_fetch}.}
\begin{lstlisting}[style=autotrajcode]
{
  "method": "POST",
  "url": "https://gmail.googleapis.com/gmail/v1/users/me/messages/msg_rev_005/modify",
  "body": "{\"removeLabelIds\": [\"UNREAD\"]}"
}
\end{lstlisting}
\end{autotrajagent}

\begin{autotrajenv}{13}{}
\textbf{Call 1.}
\begin{lstlisting}[style=autotrajoutput]
{"id": "msg_rev_002", "threadId": "thr_rev_002", "labelIds": ["INBOX"]}
\end{lstlisting}

\textbf{Call 2.}
\begin{lstlisting}[style=autotrajoutput]
{"id": "msg_rev_003", "threadId": "thr_rev_003", "labelIds": ["INBOX"]}
\end{lstlisting}

\textbf{Call 3.}
\begin{lstlisting}[style=autotrajoutput]
{"id": "msg_rev_004", "threadId": "thr_rev_004", "labelIds": ["INBOX"]}
\end{lstlisting}

\textbf{Call 4.}
\begin{lstlisting}[style=autotrajoutput]
{"id": "msg_rev_005", "threadId": "thr_rev_005", "labelIds": ["INBOX"]}
\end{lstlisting}
\end{autotrajenv}

\begin{autotrajagent}{14}{$\{\mathtt{Search}\}$}
\textbf{Thinking.} All emails are now marked as read. Now I need to reply to procurement confirming each decision. I need to send 4 separate emails (one for each vendor decision) to procurement@company.example.com.

Let me compose the reply emails. I'll need to use the Gmail API to send emails. First, let me search for the send email endpoint.

\textbf{Call 1: \texttt{api\_search}.}
\begin{lstlisting}[style=autotrajcode]
{
  "query": "gmail send email messages",
  "top_k": 5
}
\end{lstlisting}
\end{autotrajagent}

\begin{autotrajenv}{14}{ (excerpt)}
\textbf{Call 1 (excerpt).}
\begin{lstlisting}[style=autotrajoutput]
{
  "results": [
    {
      "id": "gmail.users.messages.send",
      "method": "POST",
      "description": "Compose and deliver an email to one or more recipients.",
      "url": "https://gmail.googleapis.com/gmail/v1/users/{userId}/messages/send",
      ... [other schema fields omitted]
    },
    ... [4 other search results omitted]
  ],
  "count": 5
}
\end{lstlisting}
\end{autotrajenv}

\begin{autotrajagent}{15}{$\emptyset$}
\textbf{Thinking.} I need to use \texttt{gmail.users.messages.send} to send emails. The message body needs to be base64url-encoded. Let me compose the emails and send them.

I'll send 4 separate emails to procurement@company.example.com, one for each vendor decision. The emails should confirm each decision with the vendor name and decision status.

Let me compose and send the emails:

\textbf{Call 1: \texttt{base64\_encode}.}
\begin{lstlisting}[style=autotrajcode]
{
  "text": "To: procurement@company.example.com\nSubject: Re: Vendor Review - Summit\n\nDear Procurement Team,\n\nThis is to confirm that the vendor review for Summit has been processed.\n\nDecision: Approved\n\nThe vendor card has been updated in Pipefy (tbl_ops) with:\n- Phase: phase_approved\n- Status: Approved\n\nBest regards,\nOps Team"
}
\end{lstlisting}

\textbf{Call 2: \texttt{base64\_encode}.}
\begin{lstlisting}[style=autotrajcode]
{
  "text": "To: procurement@company.example.com\nSubject: Re: Vendor Review - Apex\n\nDear Procurement Team,\n\nThis is to confirm that the vendor review for Apex has been processed.\n\nDecision: Approved\n\nThe vendor card has been updated in Pipefy (tbl_ops) with:\n- Phase: phase_approved\n- Status: Approved\n\nBest regards,\nOps Team"
}
\end{lstlisting}

\textbf{Call 3: \texttt{base64\_encode}.}
\begin{lstlisting}[style=autotrajcode]
{
  "text": "To: procurement@company.example.com\nSubject: Re: Vendor Review - NorthWind LLC\n\nDear Procurement Team,\n\nThis is to confirm that the vendor review for NorthWind LLC has been processed.\n\nDecision: Rejected\nReason: failed background check\n\nThe vendor card has been updated in Pipefy (tbl_ops) with:\n- Phase: phase_rejected\n- Status: Rejected\n\nBest regards,\nOps Team"
}
\end{lstlisting}

\textbf{Call 4: \texttt{base64\_encode}.}
\begin{lstlisting}[style=autotrajcode]
{
  "text": "To: procurement@company.example.com\nSubject: Re: Vendor Review - BlueYard\n\nDear Procurement Team,\n\nThis is to confirm that the vendor review for BlueYard has been processed.\n\nDecision: Conditional - pending insurance docs\n\nThe vendor card has been updated in Pipefy (tbl_ops) with:\n- Phase: phase_pending_docs\n- Status: Pending Docs\n\nBest regards,\nOps Team"
}
\end{lstlisting}
\end{autotrajagent}

\begin{autotrajenv}{15}{ (decoded)}
\textit{The four base64url outputs are decoded below for readability.}

\textbf{Call 1 (decoded).}
\begin{lstlisting}[style=autotrajoutput]
To: procurement@company.example.com
Subject: Re: Vendor Review - Summit

Dear Procurement Team,

This is to confirm that the vendor review for Summit has been processed.

Decision: Approved

The vendor card has been updated in Pipefy (tbl_ops) with:
- Phase: phase_approved
- Status: Approved

Best regards,
Ops Team
\end{lstlisting}

\textbf{Call 2 (decoded).}
\begin{lstlisting}[style=autotrajoutput]
To: procurement@company.example.com
Subject: Re: Vendor Review - Apex

Dear Procurement Team,

This is to confirm that the vendor review for Apex has been processed.

Decision: Approved

The vendor card has been updated in Pipefy (tbl_ops) with:
- Phase: phase_approved
- Status: Approved

Best regards,
Ops Team
\end{lstlisting}

\textbf{Call 3 (decoded).}
\begin{lstlisting}[style=autotrajoutput]
To: procurement@company.example.com
Subject: Re: Vendor Review - NorthWind LLC

Dear Procurement Team,

This is to confirm that the vendor review for NorthWind LLC has been processed.

Decision: Rejected
Reason: failed background check

The vendor card has been updated in Pipefy (tbl_ops) with:
- Phase: phase_rejected
- Status: Rejected

Best regards,
Ops Team
\end{lstlisting}

\textbf{Call 4 (decoded).}
\begin{lstlisting}[style=autotrajoutput]
To: procurement@company.example.com
Subject: Re: Vendor Review - BlueYard

Dear Procurement Team,

This is to confirm that the vendor review for BlueYard has been processed.

Decision: Conditional - pending insurance docs

The vendor card has been updated in Pipefy (tbl_ops) with:
- Phase: phase_pending_docs
- Status: Pending Docs

Best regards,
Ops Team
\end{lstlisting}
\end{autotrajenv}

\begin{autotrajagent}{16}{$\{\mathtt{Create}\}$}
\textbf{Thinking.} Now I have the base64 encoded emails. Let me send them using the Gmail API. The \texttt{gmail.users.messages.send} endpoint expects a raw RFC 2822 message in base64url format.

\textit{For display only, the \texttt{raw} payloads below are decoded. The recorded API requests contained their base64url encodings.}

\textbf{Call 1: \texttt{api\_fetch}.}
\begin{lstlisting}[style=autotrajcode]
{
  "method": "POST",
  "url": "https://gmail.googleapis.com/gmail/v1/users/me/messages/send",
  "body": "{\"raw\": \"To: procurement@company.example.com\\nSubject: Re: Vendor Review - Summit\\n\\nDear Procurement Team,\\n\\nThis is to confirm that the vendor review for Summit has been processed.\\n\\nDecision: Approved\\n\\nThe vendor card has been updated in Pipefy (tbl_ops) with:\\n- Phase: phase_approved\\n- Status: Approved\\n\\nBest regards,\\nOps Team\"}"
}
\end{lstlisting}

\textbf{Call 2: \texttt{api\_fetch}.}
\begin{lstlisting}[style=autotrajcode]
{
  "method": "POST",
  "url": "https://gmail.googleapis.com/gmail/v1/users/me/messages/send",
  "body": "{\"raw\": \"To: procurement@company.example.com\\nSubject: Re: Vendor Review - Apex\\n\\nDear Procurement Team,\\n\\nThis is to confirm that the vendor review for Apex has been processed.\\n\\nDecision: Approved\\n\\nThe vendor card has been updated in Pipefy (tbl_ops) with:\\n- Phase: phase_approved\\n- Status: Approved\\n\\nBest regards,\\nOps Team\"}"
}
\end{lstlisting}

\textbf{Call 3: \texttt{api\_fetch}.}
\begin{lstlisting}[style=autotrajcode]
{
  "method": "POST",
  "url": "https://gmail.googleapis.com/gmail/v1/users/me/messages/send",
  "body": "{\"raw\": \"To: procurement@company.example.com\\nSubject: Re: Vendor Review - NorthWind LLC\\n\\nDear Procurement Team,\\n\\nThis is to confirm that the vendor review for NorthWind LLC has been processed.\\n\\nDecision: Rejected\\nReason: failed background check\\n\\nThe vendor card has been updated in Pipefy (tbl_ops) with:\\n- Phase: phase_rejected\\n- Status: Rejected\\n\\nBest regards,\\nOps Team\"}"
}
\end{lstlisting}

\textbf{Call 4: \texttt{api\_fetch}.}
\begin{lstlisting}[style=autotrajcode]
{
  "method": "POST",
  "url": "https://gmail.googleapis.com/gmail/v1/users/me/messages/send",
  "body": "{\"raw\": \"To: procurement@company.example.com\\nSubject: Re: Vendor Review - BlueYard\\n\\nDear Procurement Team,\\n\\nThis is to confirm that the vendor review for BlueYard has been processed.\\n\\nDecision: Conditional - pending insurance docs\\n\\nThe vendor card has been updated in Pipefy (tbl_ops) with:\\n- Phase: phase_pending_docs\\n- Status: Pending Docs\\n\\nBest regards,\\nOps Team\"}"
}
\end{lstlisting}
\end{autotrajagent}

\begin{autotrajenv}{16}{}
\textbf{Call 1.}
\begin{lstlisting}[style=autotrajoutput]
{"id": "915af0dfce5442fc", "threadId": "b0e1d22d4d8448c9", "labelIds": ["SENT"]}
\end{lstlisting}

\textbf{Call 2.}
\begin{lstlisting}[style=autotrajoutput]
{"id": "4f405f1b65624897", "threadId": "15510f489d564b0d", "labelIds": ["SENT"]}
\end{lstlisting}

\textbf{Call 3.}
\begin{lstlisting}[style=autotrajoutput]
{"id": "296574d4d4524ac3", "threadId": "a513b0dcf5754c14", "labelIds": ["SENT"]}
\end{lstlisting}

\textbf{Call 4.}
\begin{lstlisting}[style=autotrajoutput]
{"id": "8820d6a75116470c", "threadId": "b534c04758364f27", "labelIds": ["SENT"]}
\end{lstlisting}
\end{autotrajenv}

\begin{autotrajagent}{17}{$\emptyset$}
\textbf{Final response (recorded Markdown).}
\begin{lstlisting}[style=autotrajoutput]
All 4 emails have been sent successfully. Let me summarize what was accomplished:

**Summary of Actions Completed:**

1. **Found 4 unread emails from procurement@company.example.com** about vendor reviews:
   - Summit - Approved
   - Apex - Approved
   - NorthWind LLC - Rejected (Reason: failed background check)
   - BlueYard - Conditional - pending insurance docs

2. **Updated Pipefy cards in tbl_ops:**
   - **Summit (card_903)**: Updated status to "Approved", moved to phase_approved
   - **Apex (card_904)**: Updated status to "Approved", moved to phase_approved
   - **NorthWind LLC (card_905)**: Updated status to "Rejected", moved to phase_rejected
   - **BlueYard (card_906)**: Updated status to "Pending Docs", moved to phase_pending_docs

3. **Marked all 4 emails as read** by removing the UNREAD label

4. **Sent confirmation emails to procurement@company.example.com** for each vendor:
   - Email 1: Summit - Approved
   - Email 2: Apex - Approved
   - Email 3: NorthWind LLC - Rejected (with reason: failed background check)
   - Email 4: BlueYard - Conditional - pending insurance docs

All tasks have been completed successfully. The vendor review decisions have been processed, Pipefy cards updated, emails marked as read, and confirmation emails sent to procurement.
\end{lstlisting}
\textit{No tool call.}
\end{autotrajagent}

\end{document}